\documentclass{article}

\usepackage{microtype}
\usepackage{graphicx}
\usepackage{subfigure}
\usepackage{booktabs} 

\usepackage{hyperref}

\usepackage[accepted]{icml2018}

\icmltitlerunning{Convergent \textsc{Tree Backup} and \textsc{Retrace} with Function Approximation}

\usepackage[utf8]{inputenc} 
\usepackage[T1]{fontenc}    
\usepackage{url}            
\usepackage{booktabs}       
\usepackage{amsfonts}       
\usepackage{nicefrac}       
\usepackage{microtype}      
\usepackage{booktabs}
\usepackage{amsmath}
\usepackage{mathrsfs}
\usepackage{amsthm}
\usepackage{stmaryrd}
\usepackage{amssymb}

\usepackage{relsize}
\usepackage{pifont}
\usepackage{multirow}
\usepackage{color}
\usepackage{mathtools}
\usepackage{tikz}
\usetikzlibrary{arrows}
\usepackage{algorithm}
\usepackage{algorithmic}
\usepackage{todonotes}

\usepackage{graphicx}

\renewcommand{\S}{\mathcal{S}}
\newcommand{\A}{\mathcal{A}}
\newcommand{\T}{\mathcal{T}}
\newcommand{\R}{\mathcal{R}}
\newcommand{\Real}{\mathbb{R}}

\newcommand{\E}{\mathbb{E}}
\newcommand{\EE}[1]{\E\left[#1\right]}
\newcommand{\cbar}{\,|\,}

\newtheorem{assumption}{Assumption}
\newtheorem{proposition}{Proposition}

\hypersetup{
	plainpages=false,
	colorlinks=true,              
	linkcolor=blue,               
	anchorcolor=blue,             
	citecolor=blue,               
	filecolor=blue,               
	pagecolor=blue,               
	urlcolor=blue,                
	pdfview=FitH,                 
	pdfstartview=FitH,            
	pdfpagelayout=SinglePage      
}

\begin{document}

\twocolumn[
\icmltitle{Convergent \textsc{Tree Backup} and \textsc{Retrace} with Function Approximation}



\icmlsetsymbol{equal}{*}
\begin{icmlauthorlist}
\icmlauthor{Ahmed Touati}{udem,fb}
\icmlauthor{Pierre-Luc Bacon}{mcgill}
\icmlauthor{Doina Precup}{mcgill,cifar}
\icmlauthor{Pascal Vincent}{udem,fb,cifar}
\end{icmlauthorlist}

\icmlaffiliation{fb}{Facebook AI Research}
\icmlaffiliation{cifar}{Canadian Institute for Advanced Research (CIFAR)}
\icmlaffiliation{udem}{MILA, Universit\'e de Montr\'eal}

\icmlaffiliation{mcgill}{MILA, McGill University}

\icmlcorrespondingauthor{Ahmed Touati}{ahmed.touati@umontreal.ca}

\icmlkeywords{Machine Learning, ICML}

\vskip 0.3in
]



\printAffiliationsAndNotice{}  

\begin{abstract}

Off-policy learning is key to scaling up reinforcement learning as it allows to learn about a target policy from the experience generated by a different behavior policy. Unfortunately, it has been challenging to combine off-policy learning with function approximation and multi-step bootstrapping in a way that leads to both stable and efficient algorithms. In this work, we show that the \textsc{Tree Backup} and \textsc{Retrace} algorithms are unstable with linear function approximation, both in theory and in practice with specific examples. Based on our analysis, we then derive stable and efficient gradient-based algorithms using a quadratic convex-concave saddle-point formulation. By exploiting the problem structure proper to these algorithms, we are able to provide convergence guarantees and finite-sample bounds. The applicability of our new analysis also goes beyond \textsc{Tree Backup} and \textsc{Retrace} and allows us to provide new convergence rates for the GTD and GTD2 algorithms without having recourse to projections or Polyak averaging. 
\end{abstract}

\section{Introduction}
Rather than being confined to their own stream of 
experience, off-policy learning algorithms are capable of 
leveraging data from a different behavior than the one being followed, which can provide many benefits: efficient parallel exploration as in~\citet{mnih2016asynchronous} and~\citet{wang2016sample}, reuse of past experience with experience replay~\citep{Lin1992} and, in many practical contexts, learning form data produced by policies that are currently deployed, but which we want to improve (as in many scenarios of working with an industrial or health care partner).
Moreover, a single stream of experience can be used to learn about a variety of different targets which may take the form of value functions corresponding to different policies and time scales~\citep{sutton1999opt} or to predicting  different reward functions as in~\citet{Sutton2004} and~\citet{sutton2011horde}. Therefore, the design and analysis of off-policy algorithms using all the features of reinforcement learning, e.g. bootstrapping, multi-step updates (eligibility traces), and function approximation has been explored extensively over three decades. While off-policy learning and function approximation have been understood in isolation, their combination with multi-steps bootstrapping produces a so-called \textit{deadly triad}~\citep{Sutton2015talk,SuttonBarto2017}, i.e., many algorithms in this category are unstable.

A convergent approach to this triad is provided by importance sampling, which bends the behavior policy distribution onto the target one~\citep{precup2000eligibility, precup2001eligibility}. However, as the length of the trajectories increases, the variance of importance sampling corrections tends to become very large. The \textsc{Tree Backup} algorithm~\citep{precup2000eligibility} is an alternative approach which remarkably does not rely on importance sampling ratios directly. More recently, \citet{munos2016safe} introduced the \textsc{Retrace} algorithm which also builds on \textsc{Tree Backup} to perform off-policy learning without importance sampling. 

Until now, \textsc{Tree Backup} and \textsc{Retrace}($\lambda$) had only been shown to converge in the tabular case, and their behavior with linear function approximation was not known. In this paper, we show that this combination with linear function approximation is in fact divergent. We obtain this result by analyzing the mean behavior of \textsc{Tree Backup} and \textsc{Retrace} using the ordinary differential equation (ODE)~\citep{Borkar2000} associated with them. We also demonstrate this instability with a concrete counterexample.  

Insights gained from this analysis allow us to derive a new gradient-based algorithm with provable convergence guarantees. Instead of adapting the derivation of Gradient Temporal Difference (GTD) learning from ~\citep{sutton2009convergent}, we use a primal-dual saddle point formulation \citep{liu2015finite, macua2015distributed} which facilitates the derivation of sample complexity bounds. 
The underlying saddle-point problem combines the primal variables, function approximation parameters, and dual variables through a bilinear term. 

In general, stochastic primal-dual gradient algorithms like the ones derived in this paper can be shown to achieve $O(1/k)$ convergence rate (where $k$ is the number of iterations). For example, this has been established for the class of forward-backward algorithms with added noise ~\citep{rosasco2016stochastic}. Furthermore, this work assumes that the objective function is composed of a convex-concave term and a strongly convex-concave regularization term that admits a tractable proximal mapping. In this paper, we are able to achieve the same $O(1/k)$ convergence rate without having to assume strong convexity with respect to the primal variables and in the absence of proximal mappings. As corollary, our convergence rate result extends to the well-known gradient-based temporal difference algorithms GTD~\citep{sutton2009convergent} and GTD2~\citep{sutton2009fast} and hence improves the previously published results.

The algorithms resulting from our analysis are simple to implement, and perform well in practice compared to other existing multi-steps off-policy learning algorithms such as GQ($\lambda$)~\citep{maei2010gq} and \textsc{AB-Trace}($\lambda$)~\citep{mahmood2017multi}.

\section{Background and notation}

In reinforcement learning, an agent interacts with its environment which we model as discounted Markov Decision Process $(\S, \A, \gamma, P, r)$ with state space $\S$, action space $\A$, discount factor $\gamma \in [0, 1)$, transition probabilities $P : \S \times \A \rightarrow (\S \rightarrow [0,1])$ mapping state-action pairs to distributions over next states, and reward function $r : (\S \times \A) \rightarrow \Real$. For simplicity, we assume the state and action space are finite, but our analysis can be extended to the countable or continuous case. We denote by $\pi(a \cbar s)$ the probability of choosing action $a$ in state $s$ under the policy $\pi : \S \rightarrow (\A \rightarrow [0, 1])$. The action-value function for policy $\pi$, denoted $Q^{\pi}:\S \times \A \rightarrow \Real$, represents the expected sum of discounted rewards along the trajectories induced by the policy in the MDP: $Q^{\pi}(s, a) = \EE{\sum_{t=0}^{\infty} \gamma^t r_t \cbar (s_0, a_0)=(s, a), \pi}$.  $Q^\pi$ can be obtained as the fixed point of the Bellman operator over the action-value function $\T^\pi Q = r + \gamma P^\pi Q$ where $r$ is the expected immediate reward and $P^\pi$ is defined as:
\begin{equation*}
(P^\pi Q)(s,a) \triangleq \sum_{s' \in \S} \sum_{a' \in \A} P(s' \cbar s, a) \pi(a' \cbar s') Q(s', a') \enspace .
\end{equation*}
In this paper, we are concerned with the policy evaluation problem~\citep{sutton1998introduction} under model-free off-policy learning. That is, we will evaluate a \textit{target} policy $\pi$ using trajectories (i.e. sequences of states, actions and rewards) obtained from a different \textit{behavior} policy $\mu$. In order to obtain generalization between different state-action pairs, $Q^\pi$ should be represented in a functional form. In this paper, we focus on linear function approximation of the form:
\begin{equation*}
Q(s, a) \triangleq \theta^{\top}  \phi(s, a) \enspace ,
\end{equation*}
where $\theta \in \Theta \subset \Real^d$ is a weight vector and $\phi: \S \times \A \rightarrow \Real^d$ is a feature map from a state-action pairs to a given $d$-dimensional feature space. 

\paragraph{Off-policy learning}

~\cite{munos2016safe} provided a unified perspective on several off-policy learning algorithms, namely: those using explicit importance sampling corrections ~\citep{precup2000eligibility} as well as  \textsc{Tree Backup} (TB($\lambda$))~\citep{precup2000eligibility} and $Q(\lambda)^{\pi}$~\citep{harutyunyan2016q} which do not involve importance ratios. As a matter of fact, all these methods share a general form based on the $\lambda$-return~\citep{SuttonBarto2017} but involve different coefficients $\kappa_i$ in :
\begin{align*}
G_k^{\lambda} & \triangleq Q(s_k, a_k) + \sum_{t=k}^{\infty} (\lambda \gamma)^{t-k} \left(\prod_{i=k+1}^t \kappa_i\right) \\
& \quad \times (r_t +  \gamma \E_\pi Q(s_{t+1}, \cdot) - Q(s_t, a_t)) \\
& =  Q(s_k, a_k) + \sum_{t=k}^{\infty} (\lambda \gamma)^{t-k} \left(\prod_{i=k+1}^t \kappa_i\right) \delta_t \enspace,
\end{align*}
where $\E_\pi Q(s_{t+1}, .) \triangleq \sum_{a \in \A} \pi(a \cbar s_{t+1})Q(s_{t+1}, a)$ and $\delta_t \triangleq r_t +  \gamma \E_\pi Q(s_{t+1}, .) - Q(s_t, a_t)$ is the temporal-difference (TD) error. The coefficients $\kappa_i$ determine how the TD errors would be scaled in order to correct for the discrepancy between target and behavior policies. From this unified representation,~\citet{munos2016safe} derived the \textsc{Retrace}($\lambda$) algorithm. Both TB($\lambda$) and \textsc{Retrace}($\lambda$) consider this form of return, but set $\kappa_i$ differently. The TB($\lambda$) updates correspond to the choice $\kappa_i = \pi(a_i \cbar s_i)$ while \textsc{Retrace}($\lambda$) sets $\kappa_i = \min \left(1, \frac{\pi(a_i \cbar s_i)}{\mu(a_i \cbar s_i)} \right)$, which is intended to allow learning from full returns when the target and behavior policies are very close.
The importance sampling approach~\citep{precup2000eligibility} converges in the tabular case by correcting the behavior data distribution to the distribution that would be induced by the target policy $\pi$. However, these correction terms lead to high variance in practice. Since $Q(\lambda)$ does not involve importance ratios, this variance problem is avoided but at the cost of restricted convergence guarantees satisfied  only when the behavior and target policies are sufficiently close.

The analysis provided in this paper concerns TB($\lambda$) and \textsc{Retrace}($\lambda$), which are convergent in the tabular case, but have not been analyzed in the function approximation case. We start by noting that the Bellman operator \footnote{We overload our notation over linear operators and their corresponding matrix representation.}  $\mathcal{R}$ underlying these these algorithms can be written in the following form: 
\begin{align*}
(\R Q)(s,a) & \triangleq Q(s, a) + \E_{\mu} \Big[ \sum_{t=0}^{\infty} (\lambda \gamma)^{t} \left(\prod_{i=1}^t \kappa_i\right)  \\
& \quad \times (r_t+ \gamma \E_\pi Q(s_{t+1}, \cdot) - Q(s_t, a_t)) \Big] \\
& = Q(s,a) + (I- \lambda \gamma P^{\kappa \mu})^{-1} (\T^{\pi} Q - Q)(s, a) \enspace,
\end{align*}
where $\E_{\mu}$ is the expectation over the behavior policy and MDP transition probabilities and $P^{\kappa \mu}$ is the operator defined by:
\begin{equation*}
(P^{\kappa \mu} Q)(s, a) \triangleq \sum_{ \substack{s' \in \S \\ a' \in \A}} P(s' \cbar s, a) \mu(a' \cbar s') \kappa(s', a') Q(s', a') \enspace .
\end{equation*}

In the tabular case, these operators were shown to be contraction mappings with respect to the max norm ~\citep{precup2000eligibility,munos2016safe}. In this paper, we focus on what happens to these operators when combined with linear function approximation.

\section{Off-policy instability with function approximation}
\label{sect:offpolicyinstability}
When combined with function approximation, the temporal difference updates corresponding to the $\lambda$-return $G_k^\lambda$ are given by 
\begin{align}
\theta_{k+1} & = \theta_k + \alpha_k \left(G_k^{\lambda} - Q(s_k, a_k) \right)\nabla_{\theta} Q(s_k, a_k) \notag \\
& = \theta_k + \alpha_k \left(\sum_{t=k}^{\infty} (\lambda \gamma)^{t-k} \left(\prod_{i=k+1}^t \kappa_i\right) \delta_t^k \right)\phi(s_k, a_k)\label{eq:theta}
\end{align}
where $\delta_t^k = r_t + \gamma \theta_k^{\top}\E_{\pi}\phi(s_{t+1}, \cdot) - \theta_k^{\top} \phi(s_t, a_t)$ and $\alpha_k$ are positive non-increasing step sizes. The updates~\eqref{eq:theta} implies off-line updating as $G_k^\lambda$ is a quantity which depends on future rewards. This will be addressed later using eligibility traces: a mechanism to transform the off-line updates into efficient on-line ones. Since~\eqref{eq:theta} describes stochastic updates, the following standard assumption is necessary:

\begin{assumption}
The Markov chain induced by the behavior policy $\mu$ is ergodic and admits a unique stationary distribution, denoted by $\xi$, over state-action pairs. We write $\Xi$ for the diagonal matrix whose diagonal entries are $(\xi(s,a))_{s \in \S, a \in \A}$.
\label{ergodic}
\end{assumption}

Our first proposition establishes the expected behavior of the parameters in the limit.
\begin{proposition}\label{prop:function_approx}
If the behavior policy satisfies Assumption~\ref{ergodic} and $(\theta_k)_{k\leq 0}$ is the Markov process  defined by~\eqref{eq:theta} then:
\begin{equation*}
\E[\theta_{k+1} \cbar \theta_0] = \left(I + \alpha_k A \right)\E[\theta_k \cbar \theta_0]  + \alpha_k b \enspace ,
\end{equation*}
where matrix $A$ and vector $b$ are defined as follows:
\begin{equation*}
\begin{aligned}
A & \triangleq \Phi^\top \Xi(I - \lambda \gamma \mathit{P}^{\kappa \mu})^{-1}(\gamma \mathit{P}^{\pi} - I) \Phi \enspace, \\
b & \triangleq \Phi^\top \Xi(I - \lambda \gamma P^{\kappa \mu})^{-1}r \enspace .
\end{aligned}
\end{equation*}

\end{proposition}

\begin{proof}[Sketch of Proof (The full proof is in the appendix)]
\begin{equation*}
\begin{aligned}
\theta_{k+1}
& = \theta_k + 
\alpha_k \Big(\sum_{t=k}^{\infty} (\lambda \gamma)^{t-k} \left(\prod_{i=k+1}^t \kappa_i\right) \phi(s_k, a_k)  \\
& \quad \times  \left( [\gamma \mathbb{E}_{\pi} \phi(x_{t+1}, \cdot) - \phi(x_t, a_t)]^\top \theta_k 
  + r_t \right) \Big) \\
& = \theta_k + \alpha_k \left(A_k \theta_k+ b_k \right) \enspace .
\end{aligned}
\end{equation*}
So, $\E[\theta_{k+1} \cbar \theta_k] = (I + \alpha_k A) \theta_k + \alpha_k b$ where $A =\E[A_k]$ and $b = \E[b_k]$ 
\end{proof}

The ODE (Ordinary Differential Equations) approach~\citep{Borkar2000} is the main tool to establish convergence in the function approximation case~\citep{bertsekas1995neuro,tsitsiklis1997analysis}. In particular, we use  Proposition 4.8 in ~\citet{bertsekas1995neuro}, which states that under some conditions, $\theta_k$ converges to the unique solution $\theta^*$ of the system $A \theta^* + b = 0$. This crucially relies on the matrix $A$ being negative definite i.e $ y^\top A y < 0, \forall y \neq 0$. In the on-policy case, when $\mu = \pi$, we rely on the fact that the stationary distribution is invariant under the transition matrix $P^{\pi}$ i.e $d^\top P^{\pi} = d^\top$~\citep{tsitsiklis1997analysis,sutton2015emphatic}. However, this is no longer true for off-policy learning with arbitrary target/behavior policies and the matrix $A$ may not be negative definite: the series $\theta_k$ may then diverge.  We will now see that the same phenomenon may occur with TB($\lambda$) and \textsc{Retrace}($\lambda$).

\paragraph{Counterexample:} We extend the two-states MDP of~\citet{tsitsiklis1997analysis}, originally proposed to show the divergence of off-policy TD(0), to the case of function approximation over state-action pairs. This environment has only two states, as shown in Figure~\ref{fig:example}, and two actions: left or right.  

\begin{figure}[H]
\centering
\begin{tikzpicture}[->,>=stealth',auto,node distance=3cm,
  thick,main node/.style={circle,draw,font=\sffamily\Large\bfseries}]

  \node[main node] (1) {1};
  \node[main node] (2) [right of=1] {2};

  \path[every node/.style={font=\sffamily\small}]
    (1) edge [bend right] node [right] {} (2)
    (2) edge [bend right] node [right] {} (1)
    (2) edge [loop] node {} (2)
    (1) edge [loop] node {} (1);
\end{tikzpicture}
\caption{\label{fig:example} Two-state counterexample. We assign the features 
$\{ (1,0)^\top, (2, 0)^\top, (0, 1)^\top, (0, 2)^\top \}$ to the state-action pairs $\{ (1, \text{right}), (2, \text{right}), (1, \text{left}), (2, \text{left})\}$. The target policy is given by $\pi(\text{right} \cbar \cdot) = 1 $ and the behavior policy is $\mu(\text{right} \cbar \cdot) = 0.5$}
\end{figure}
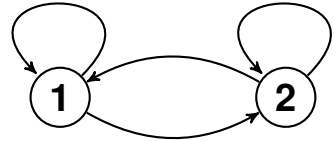
In this particular case, both TB($\lambda$) and \textsc{Retrace}($\lambda$) share the same matrix $P^{\kappa \mu}$
and $P^{\kappa \mu} = 0.5 P^{\pi}$:
\begin{equation*}
P^{\pi} =  \left (
\begin{matrix}
0 & 1 & 0 & 0\\
0 & 1 & 0 & 0 \\
1 & 0 & 0 & 0 \\
1 & 0 & 0 & 0
\end{matrix}
\right ), 
(P^{\pi})^{n} = \left (
\begin{matrix}
0 & 1 & 0 & 0\\
0 & 1 & 0 & 0 \\
0 & 1 & 0 & 0 \\
0 & 1 & 0 & 0
\end{matrix}
\right ) \forall n \geq 2
\end{equation*}
If we set $\beta := 0.5 \gamma \lambda$, we then have:
\begin{align*}
(I - \lambda \gamma P^{\kappa \mu} )^{-1}
 & = 
\left (
\begin{matrix}
1 & \frac{\beta}{1-\beta} & 0 & 0\\
0 & \frac{1}{1-\beta}  & 0 & 0 \\
\beta & \frac{\beta^2}{1-\beta}   & 1 & 0 \\
\beta & \frac{\beta^2 }{1-\beta} & 0 & 1
\end{matrix}
\right ), \\
A & = \left(
\begin{matrix}
\frac{6 \gamma - \beta -5}{1- \beta}& 0 \\
\frac{3 (\gamma \beta - \beta^2 - \beta - \gamma)}{1-\beta} & -5 \\
\end{matrix}\right)\enspace .
\end{align*}
Therefore, $\forall \, \gamma \in (\frac{5}{6}, 1)$ and $\forall \, \lambda \in [0, \min(1, \frac{12 \gamma -10}{\gamma}))$, the first eigenvalue $e_1 = \frac{6 \gamma - \beta -5}{1- \beta}$ of $A$ is positive. The basis vectors $(1,0)^\top$ and $(0,1)^\top$ are eigenvectors of A associated with $e_1$ and -5, then if $\theta_0 = (\eta_1, \eta_2)^\top$,
we obtain $
\E[\theta_{k} \cbar \theta_0] = (\eta_1\prod_{i=0}^{k-1}(1 + \alpha_i  e_1), \eta_2 \prod_{i=0}^{k-1}(1 -5 \alpha_i))^\top  $
implying that $|| \E[\theta_{k} \cbar \theta_0] || \geq |\eta_1| \prod_{i=0}^{k-1}(1 + \alpha_i e_1)$. Hence, as $\sum_{k}\alpha_k \rightarrow \infty$, $|| \E[\theta_{k} \cbar \theta_0] || \rightarrow \infty$ if $\eta_1 \neq 0$.

\section{Convergent gradient off-policy algorithms}

If $A$ were to be negative definite, \textsc{Retrace}($\lambda$) or TB($\lambda$) with function approximation would converge to $\theta^*= -A^{-1}b$. It is known~\citep{bertsekas2011temporal} that $\Phi \theta^*$ is the fixed point of the projected Bellman operator :
\begin{equation*}
\Phi \theta^* = \Pi^{\mu} \mathcal{R}(\Phi \theta^*) \enspace ,
\end{equation*}
where $\Pi^{\mu} = \Phi (\Phi^\top \Xi \Phi)^{-1}\Phi^\top \Xi$ is the orthogonal projection onto the space $S = \{ \Phi \theta | \theta \in \mathbb{R}^d \}$ with respect to the weighted Euclidean norm $||.||_{\Xi}$. 
Rather than computing the sequence of iterates given by the projected Bellman operator, another approach for finding 
$\theta^*$ is to directly minimize~\citep{Sutton2009,liu2015finite} the Mean Squared Projected Bellman Error (MSPBE):
\begin{equation*}
\begin{aligned}
\mathbf{MSPBE}(\theta) = \frac{1}{2}||\Pi^{\mu} \mathcal{R}(\Phi \theta)- \Phi \theta ||^2_{\Xi} \enspace .
\end{aligned}
\end{equation*}

This is the route that we take in this paper to derive convergent forms of TB($\lambda$) and \textsc{Retrace}($\lambda$).
To do so, we first define our objective function in terms of $A$ and $b$ which we introduced in Proposition \ref{prop:function_approx}.
\begin{proposition}\label{prop:MSPBE}
Let $M \triangleq \Phi^\top \Xi \Phi = \E[\Phi \Phi^\top]$ be the covariance matrix of features. We have:
\begin{equation*}
\mathbf{MSPBE}(\theta) = \frac{1}{2} || A \theta + b||^2_{M^{-1}}
\end{equation*}
(The proof is provided in the appendix.)
\end{proposition}

In order to derive parameter updates, we could compute gradients of the above expression explicitly as in~\citet{sutton2009convergent}, but we would then obtain a gradient that is a product of expectations. The implied double sampling makes it difficult to obtain an unbiased estimator of the gradient.~\citet{sutton2009convergent} addressed this problem with a two-timescale stochastic approximations. However, the algorithm obtained in this way is no longer a true stochastic gradient method with respect to the original objective.
~\citet{liu2015finite} suggested an alternative which converts the original minimization problem into a primal-dual saddle-point problem. This is the approach that we chose in this paper.

The convex conjugate of a real-valued function $f$ is defined as:
\begin{equation}
f^*(y)  = \sup_{x \in \mathit{X}} ( \langle y, x \rangle  - f(x))\enspace ,
\end{equation}
and $f$ is convex, we have $f^{**} = f$.
Also, if $f(x) = \frac{1}{2} ||x||_{M^{-1}}$, then $f^{*}(x) = \frac{1}{2}||x||_{M}$. Note that by going to the convex conjugate, we do not need to invert matrix $M$. 
We now go back to the original minimization problem:
\begin{equation*}
\begin{aligned}
\min_{\theta}\mathbf{MSPBE}(\theta) & \Leftrightarrow \min_{\theta} \frac{1}{2} || A \theta + b||_{M^{-1}}^2 \\
& \Leftrightarrow \min_{\theta} \max_{\omega} \left( \langle A\theta+b, \omega \rangle - \frac{1}{2} || \omega||_{M}^2 \right)
\end{aligned}
\end{equation*}
The gradient updates resulting from the saddle-point problem (ascent in $\omega$ and descent in $\theta$) are then:
\begin{equation}
\begin{aligned}
\omega_{k+1} & = \omega_k + \eta_k (A\theta_k + b - M\omega_k) \enspace, \\
\theta_{k+1} & = \theta_k - \alpha_k A^\top \omega_k \enspace.
\end{aligned}
\label{eq:saddle}
\end{equation}
where $\{\eta_k\}$ and $\{\alpha_k\}$ are non-negative step-size sequences.
As the $A$, $b$ and $M$ are all expectations, we can derive stochastic updates by drawing samples, which would yield unbiased estimates of the gradient.

\paragraph{On-line updates:} We now derive on-line updates  by exploiting equivalences in expectation between forward views and backward views outlined in~\citet{maei2011gradient}. 
\begin{proposition} \label{prop:eligibility}
Let $e_k$ be the eligibility traces vector, defined as $e_{-1} = 0$ and :
\begin{equation*}
e_k  = \lambda \gamma \kappa(s_k, a_k) e_{k-1} + \phi(s_k, a_k) \quad \forall k \geq 0\enspace .
\end{equation*}
Furthermore, let 
$\hat{A}_k = e_k (\gamma \mathbb{E}_{\pi}[\phi(s_{k+1}, .)] - \phi(s_k, a_k)])^\top,  \quad
\hat{b}_k  =  r(s_k, a_k) e_k,  \quad
\hat{M}_k  = \phi(s_k, a_k) \phi(s_k, a_k)^\top
$.
Then, we have $\E[\hat{A}_k] = A$, $\E[\hat{b}_k] = b$ and $\E[\hat{M}_k] = M$.

(The proof is provided in the appendix.)
\end{proposition}

This proposition allows us to replace the expectations in Eq.~\eqref{eq:saddle} by corresponding unbiased estimates. The resulting detailed procedure is provided in Algorithm~\ref{algo:eligibility}.

\begin{algorithm}
\caption{\label{algo:eligibility}Gradient Off-policy with eligibility traces }
\begin{algorithmic}
\medskip
\item[\textbf{Given:}] target policy $\pi$, behavior policy $\mu$
\STATE Initialize $\theta_0$ and $\omega_0$
\FOR {n = 0 \ldots}	
	\STATE set $e_0 = 0$
  	\FOR {k = 0 \ldots  end of episode}
    	\STATE Observe $s_k, a_k, r_k, s_{k+1}$ according to $\mu$
    	\STATE \textbf{Update traces}
    	\STATE $e_k = \lambda \gamma \kappa(s_k, a_k) e_{k-1} + \phi(s_k, a_k)$
    	\STATE \textbf{Update parameters}
    	\STATE $\delta_k = r_k + \gamma \theta_{k}^\top \mathbb{E}_{\pi}\phi(s_{k+1}, .) - \theta_{k}^\top\phi(s_k, a_k)$
    	\STATE $\omega_{k+1}  = \omega_{k}+ \eta_k \left( \delta_k e_k 
    - \omega_{k}^\top \phi(s_k, a_k) \phi(s_k, a_k)\right)$
    	\STATE $\theta_{k+1} = \theta_{k} - \alpha_k \omega_{k}^\top e_k \left(\gamma \mathbb{E}_{\pi}\phi(s_{k+1}, .) - \phi(s_k, a_k)\right) $
  	\ENDFOR
\ENDFOR
\end{algorithmic}
\end{algorithm}


\section{Convergence Rate Analysis}
In order to characterize the convergence rate of the algorithm \ref{algo:eligibility}, we need to introduce some new notations and state new assumptions.

We denote by $\|A\| \triangleq \sup_{\|x\|=1}\| Ax\|$ the spectral norm of the matrix A and by $c(A) = \|A\| \| A^{-1}\|$ its condition number. If the eigenvalues of a matrix $A$ are real, we use $\lambda_{\max}(A)$ and $\lambda_{\min}(A)$ to denote respectively the largest and the smallest eigenvalue.

If we set $\eta_k = \beta \alpha_k$ for a positive constant $\beta$, it is possible to combine the two iterations present in our algorithm as a single iteration using a parameter vector $z_k \triangleq  \left( \begin{matrix}
\theta_k \\
\frac{1}{\sqrt{\beta}}\omega_k
\end{matrix} \right) $ where : 
\begin{equation*}
    z_{k+1} = z_k - \alpha_k (\hat{G}_k z_k - \hat{g}_k)
\end{equation*}
where:
\begin{equation*}
\hat{G}_k \triangleq \left( \begin{matrix}
0 & \sqrt{\beta} \hat{A}_k^\top \\
- \sqrt{\beta} \hat{A}_k & \beta \hat{M}_k
\end{matrix}
\right) 
\quad 
\hat{g}_k \triangleq  \left( \begin{matrix}
0 \\
\sqrt{\beta} \hat{b}_k
\end{matrix} \right) 
\end{equation*}
Let $G \triangleq \E\left[ \hat{G}_k\right]$ and $g = \E\left[ \hat{g}_k\right]$. It follows from the proposition \ref{prop:eligibility} that $G$ and $g$ are well defined and more specifically:
\begin{equation*}
G = \left( \begin{matrix}
0 & \sqrt{\beta} A^\top \\
- \sqrt{\beta} A & \beta M
\end{matrix} \right) 
\quad 
g = \left( \begin{matrix}
0 \\
\sqrt{\beta} b
\end{matrix} \right)
\end{equation*}
Furthermore, let $\mathcal{F}_k = \sigma(z_0, \hat{G}_0, \hat{g}_0 \ldots, z_k, \hat{G}_k, \hat{g}_k, z_{k+1})$ be the sigma-algebra generated by the variables up to time $k$. With these definitions, we can now state our assumptions.

\begin{assumption} \label{assump:non-singularity} The matrices $A$ and $M$ are nonsingular. This implies that the saddle-point problem admits a unique solution $(\theta^{\star}, \omega^{\star}) = (-A^{-1}b, 0)$ and we define $z^{\star} \triangleq (\theta^{\star}, \frac{1}{\sqrt{\beta}}\omega^{\star})$. 
\end{assumption}
\begin{assumption} \label{assump:boundness} The features and reward functions are uniformly bounded. This implies that the features and rewards have uniformly bounded second moments. It follows that there exists a constant $\sigma$ such that:
\begin{equation*}
    \E[\| \hat{G}_k z_k - \hat{g}_k\|^2 | \mathcal{F}_{k-1}] \leq \sigma^2 (1 + \| z_k\|^2)
\end{equation*}
\end{assumption}
Before stating our main result, the following key quantities needs to be defined:
\begin{equation*}
\rho \triangleq \lambda_{\max}(A^{\top}M^{-1}A), \quad
\delta \triangleq \lambda_{\min}(A^{\top} M^{-1}A ),
\end{equation*}
\begin{equation*}
L_G \triangleq \Big\| \E \left[ \hat{G}_k^{\top} \hat{G}_k\cbar \mathcal{F}_{k-1} \right] \Big\|
\end{equation*}
The following proposition characterize the convergence in expectation of $\| z_k - z^{\star}\|^2 = \| \theta_k - \theta^{\star} \|^2 + \frac{1}{\beta} \|w_k\|^2$ 

\begin{proposition} \label{prop:convergence-rate}
Suppose assumptions \ref{assump:non-singularity} and \ref{assump:boundness} holds and
if we choose $\beta = \frac{8 \rho}{\lambda_{\min}(M)}$ and $\alpha_k
= \frac{9^2 \times 2 \delta}{8 \delta^2 (k+2) + 9^2 \zeta}$ where $\zeta =2 \times 9^2 c(M)^2 \rho^2
+ 32 c(M) L_G$. Then the mean square error $\E \left[ \| z_k - z^{\star} \|^2\right]$ is upper bounded by:
$$
    9^2 \times 8 c(M) \Big\{\frac{(8\delta + 9\zeta)^2 \E \left[ \| z_{0} - z^{\star} \|^2\right] }{(8^2\delta^2 k + 9^2\zeta)^2 }
+ \frac{8 \sigma^2 (1 + \| z^{\star}\|^2)}{  (8^2\delta^2 k + 9^2\zeta) } \Big\}
$$

\end{proposition}
\begin{proof}[Sketch of Proof (The full proof is in the appendix)]
The beginning of our proof relies on ~\citet{du2017stochastic} which shows the linear convergence rate of deterministic primal-dual gradient method for policy evaluation. More precisely, we make use of the spectral properties of matrix $G$ shown in the appendix of this paper. The rest of the proof follows a different route exploiting the structure of our problem.
\end{proof}
The above proposition \ref{prop:convergence-rate} shows that the mean square error $\E \left[ \| z_k - z^{\star} \|^2\right]$ at iteration $k$ is upper bounded by  tow terms. The first bias term tells that the initial error $\E \left[ \| z_0 - z^{\star} \|^2\right]$ is forgotten at a rate $O(1/k^2)$ and the constant depends on the condition number of the covariance matrix $c(M)$. The second variance term shows that noise is rejected at a rate $O(1/k)$ and the constant depends on the variance of estimates $\sigma^2$ and $c(M)$. The overall convergence rate is $O(1/k)$.

\paragraph{Existing stochastic saddle-point problem results:}

\citet{chen2014optimal} provides a comprehensive review of stochastic saddle-point problem. When the objective function is convex-concave, the overall convergence rate is $O(1/\sqrt{k})$. Although several accelerated techniques could improve the dependencies on the smoothness constants of the problem in their convergence rate, the dominant term that depends on the gradient variance still decays only as $O(1/\sqrt{k})$.

When the objective function is strongly convex-concave, ~\citet{rosasco2016stochastic} and~\citet{palaniappan2016stochastic} showed that stochastic forward-backward algorithms can achieve $O(1/k)$ convergence rate. Algorithms in this class are feasible in practice only if their proximal mappings can be computed efficiently. In our case, our objective function is strongly concave because of the positive-definiteness of $M$ but is otherwise not strongly convex. Because our algorithms are vanilla stochastic gradient methods, they do not rely on proximal mappings.

\paragraph{Singularity:} If assumption \ref{assump:non-singularity} does not hold, the matrix $G$ is singular and either $G z + g = 0$ has infinitely many solutions or it has no solution. In the case of many solutions, we could still get asymptotic convergence. In \citet{wang2013stabilization}, it was shown that under some assumptions on the null space of matrix $G$ and using a simple stabilization scheme, the iterates converge to the Drazin \citep{Drazin1958} inverse solution of $Gz+g = 0$. However, it is not clear how extend our finite-sample analysis because the spectral analysis of the matrix $G$ \citep{benzi2006eigenvalues} in our proof assumes that the matrices $A$ and $M$ are nonsingular.

\section{Related Work and Discussion}

\begin{table*} 
\begin{center}
\centerline{
\resizebox{1.\textwidth}{!}{
\begin{tabular}{|p{4cm} p{5cm} p{4cm} p{4cm} p{5cm}| }
 \hline
 Paper & step-sizes & Projection &  Polyak averaging & Convergence rate\\
 \hline
 ~\citet{sutton2009convergent}, ~\citet{sutton2009fast} & $\eta_k = \beta \alpha_k, \beta > 0,$ \newline
 $\sum_{k}\alpha_k = \infty, \sum_k \alpha_k^2 < \infty$ & \textcolor{green}{No} & \textcolor{green}{No} & $\theta_k \rightarrow \theta^{\star}$ with probability one \\
  \hline
 ~\citet{liu2015finite} & constant step-size, $\alpha_k=\eta_k$ & \textcolor{red}{Yes} & \textcolor{red}{Yes} & $ \mathbf{MSPBE}(\bar{\theta}_k) \in O(1/\sqrt{k})$ with high probability\\
  \hline
 ~\citet{wang2017finite} & $\alpha_k=\eta_k, \sum_k \alpha_k = \infty, $ \newline $\frac{\sum_k \alpha_k^2}{\sum_k \alpha_k} < \infty$ & \textcolor{red}{Yes} & \textcolor{red}{Yes} & $\mathbf{MSPBE}(\bar{\theta}_k) \in O(\frac{\sum_k \alpha_k^2}{\sum_k \alpha_k})$ with high probability \\
  \hline
 ~\citet{lakshminarayanan2017linear} & constant step-size, $\alpha_k=\eta_k$ & \textcolor{green}{No} & \textcolor{red}{Yes} & $\E[\| \bar{\theta}_k - \theta^{\star}\|^2] \in O(1/k)$ \\
  \hline
  ~\citet{dalal2017finite} & $\alpha_k = \frac{1}{k^{1-c}}, \eta_k = \frac{1}{k^{(2/3)(1-c)}}$ \newline  where $c \in (0, 1)$ & \textcolor{red}{Yes} & \textcolor{green}{No} & $\| \theta_k - \theta^{\star}\| \in O(k^{-\frac{1}{3} + \frac{c}{3} })$ with high probability \\
  \hline
 Our work & $\eta_k = \beta \alpha_k, \beta > 0, \alpha_k \in O(1/k)$ & \textcolor{green}{No} & \textcolor{green}{No} & $\E[\| \theta_k - \theta^{\star}\|^2] \in O(1/k)$ \\
 \hline
\end{tabular}
} 
} 
\caption{
\label{table: rates} Convergence results for gradient-based TD algorithms shown in previous work~\citep{sutton2009fast, sutton2009convergent, liu2015finite, wang2017finite, lakshminarayanan2017linear, dalal2017finite}. $\bar{\theta}_k$ stand for the Polyak-average of iterates: $\bar{\theta}_k \triangleq \frac{\sum_k \alpha_k \theta_k}{\sum_k \alpha_k}$. Our algorithms achieve $O(1/k)$ without the need for projections or Polyak averaging.
}
\end{center}
\vspace*{-4mm}
\end{table*}

\paragraph{Convergent \textsc{Retrace}:} ~\citet{mahmood2017multi} have recently introduced the ABQ($\zeta$) algorithm which uses an action-dependent bootstrapping parameter that leads to off-policy multi-step learning without importance sampling ratios. They also derived a gradient-based algorithm called \textsc{AB-Trace}($\lambda$) which is related to \textsc{Retrace}($\lambda$). However, the resulting updates are different from ours, as they use the two-timescale approach of~\citet{Sutton2009} as basis for 
their derivation. In contrast, our approach uses the saddle-point formulation, avoiding the need for double sampling. Another benefit of this formulation is that it allows us to provide a bound of the convergence rate (proposition~\ref{prop:convergence-rate}) whereas~ \citet{mahmood2017multi} is restricted to a more general two-timescale asymptotic result from~\citet{Borkar2000}. The saddle-point formulation also provides a rich literature on acceleration methods which could be incorporated in our algorithms. Particularly in the batch setting, ~\citet{du2017stochastic} recently introduced Stochastic Variance Reduction methods for state-value estimation combining GTD with SVRG~\citet{johnson2013accelerating} or SAGA~\citet{defazio2014saga}. This work could be extended easily to ours algorithms in the batch setting. 
\paragraph{Existing Convergence Rates:} Our convergence rate result~\ref{prop:convergence-rate} can apply to GTD/GTD2 algorithms. Recall that GTD/GTD2 are off-policy algorithms designed to estimate the state-value function using temporal difference TD(0) return while our algorithms compute the action-value function using \textsc{Retrace} and \textsc{Tree Backup} returns. In both GTD and GTD2, the quantities $\hat{A}_k$ and $\hat{b}_k$  involved in their updates are the same and equal to $\hat{A}_k = \phi(s_k) (\gamma \phi(s_{k+1}) - \phi(s_k))^{\top}$, 
$\hat{b}_k = r(s_k, a_k) \phi(s_k)$ while the matrix $\hat{M}_k$ is equal to $\phi(s_k) \phi(s_k)^{\top}$ for GTD2 and to identity matrix for GTD.\\
The table \ref{table: rates} show in chronological order the convergence rates established in the literature of Reinforcement learning. GTD was first introduced in ~\citet{sutton2009convergent} and its variant GTD2 was introduced later in ~\citet{sutton2009fast}. Both papers established the asymptotic convergence with  Robbins-Monro step-sizes. Later,~\citet{liu2015finite} provided the first sample complexity  by reformulating GTD/GTD2 as an instance of mirror stochastic approximation ~\citep{nemirovski2009robust}.~\citet{liu2015finite} showed that with high probability, $ \mathbf{MSPBE}(\bar{\theta}_k) \in O(1/\sqrt{k})$ where $\bar{\theta}_k \triangleq \frac{\sum_k \alpha_k \theta_k}{\sum_k \alpha_k}$. However, they studied an alternated version of GTD/GTD2 as they added a projection step into bounded convex set and Polyak-averaging of iterates.~\citet{wang2017finite} studied also the same version as ~\citet{liu2015finite} but for the case of Markov noise case instead of the \textit{i.i.d} assumptions. They prove that with high probability $\mathbf{MSPBE}(\bar{\theta}_k) \in O(\frac{\sum_k \alpha_k^2}{\sum_k \alpha_k})$ when the step-size sequence satisfies $\sum_k \alpha_k = \infty, \frac{\sum_k \alpha_k^2}{\sum_k \alpha_k} < \infty$. The optimal rate achieved in this setup is then $O(1/\sqrt{k})$. Recently,~\citet{lakshminarayanan2017linear} improved on the existing results by showing for the first time that $\E[\| \bar{\theta}_k - \theta^{\star}\|^2] \in O(1/k)$ without projection step. However, the result still consider the Polyak-average of iterates. Moreover, the constants in their bound depend on the data distribution that are difficult to relate to the problem-specific constants, such as those present in our bound \ref{prop:convergence-rate}. Finally, \citet{dalal2017finite} studied sparsily projected version of GTD/GTD2 and they showed that for step-sizes $\alpha_k = \frac{1}{k^{1-c}}, \eta_k = \frac{1}{k^{(2/3)(1-c)}}$  where $c \in (0, 1)$, $\| \theta_k - \theta^{\star}\| \in O(k^{-\frac{1}{3} + \frac{c}{3} })$ with high probability. The projection is called sparse as they project only on iterations which are powers of 2.
\\
Our work is the first to provide a finite-sample complexity analysis of GTD/GTD2 in its original setting, i.e without assumption a projection step or Polyak-averaging and with diminishing step-sizes.

\section{Experimental Results}

\paragraph{Evidence of instability in practice:}
To validate our theoretical results about instability, we implemented TB($\lambda$), \textsc{Retrace}($\lambda$) and compared them against their gradient-based counterparts GTB($\lambda$) and G\textsc{Retrace}($\lambda$) derived in this paper. The first one is the $2$-states counterexample that we detailed in the third section and the second is the $7$-states versions of Baird's counterexample~\citep{baird1995residual}. Figures \ref{fig:baird} and \ref{fig:2-state} show the MSBPE (averaged over $20$ runs) as a function of the number of iterations. We can see that our gradient algorithms converge in these two counterexamples whereas TB($\lambda$) and \textsc{Retrace}($\lambda$) diverge.

\begin{figure}[ht]
\centering
\includegraphics[width=0.45\linewidth]{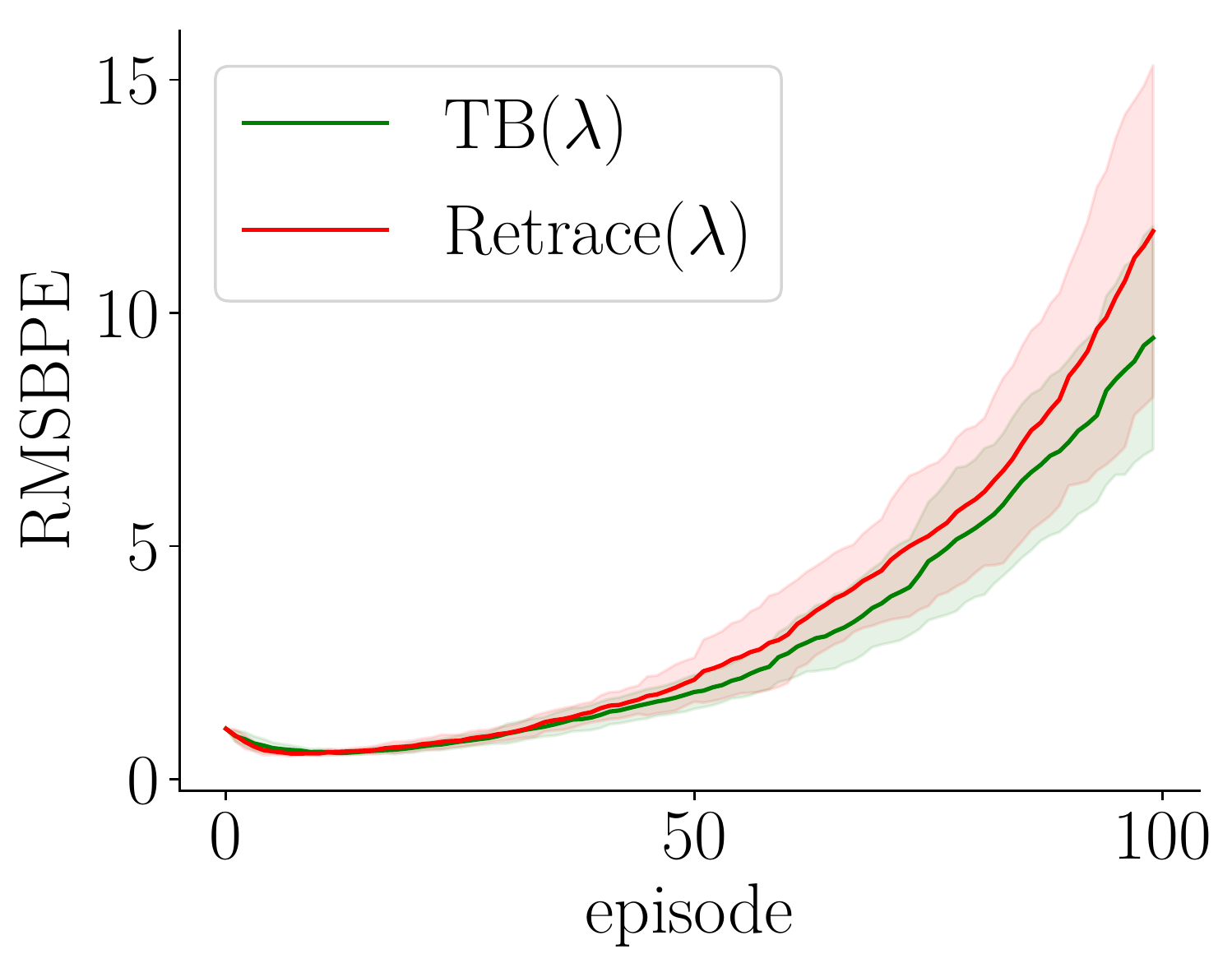}\hspace{2em}
\includegraphics[width=0.45\linewidth]{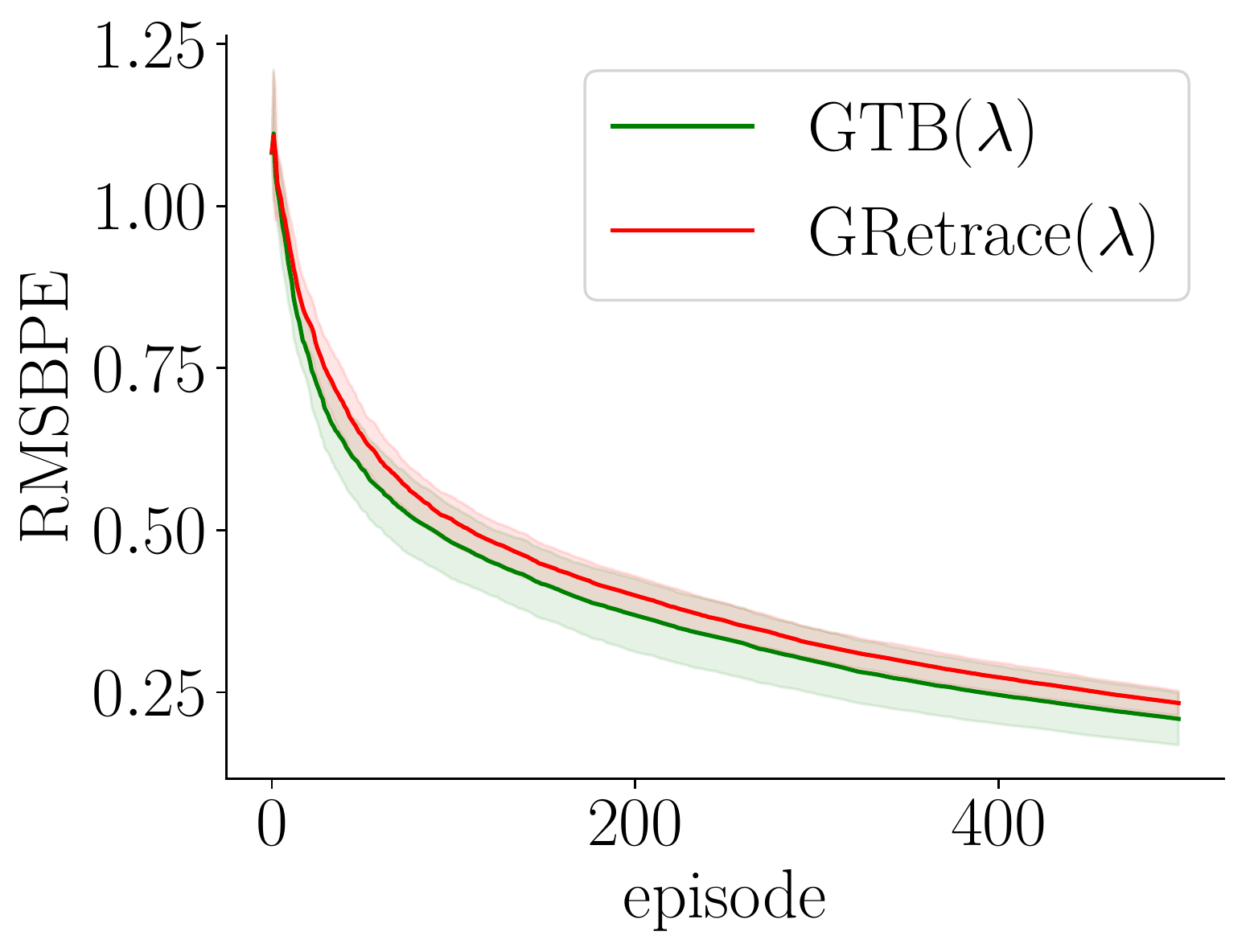}
\caption{\label{fig:baird}Baird's counterexample. The combination of linear function approximation with TB and \textsc{Retrace} leads to divergence (left panel) while the proposed gradient extensions GTB and G\textsc{Retrace} converge (right panel).}
\end{figure}

\begin{figure}[ht]
\centering
\includegraphics[width=0.45\linewidth]{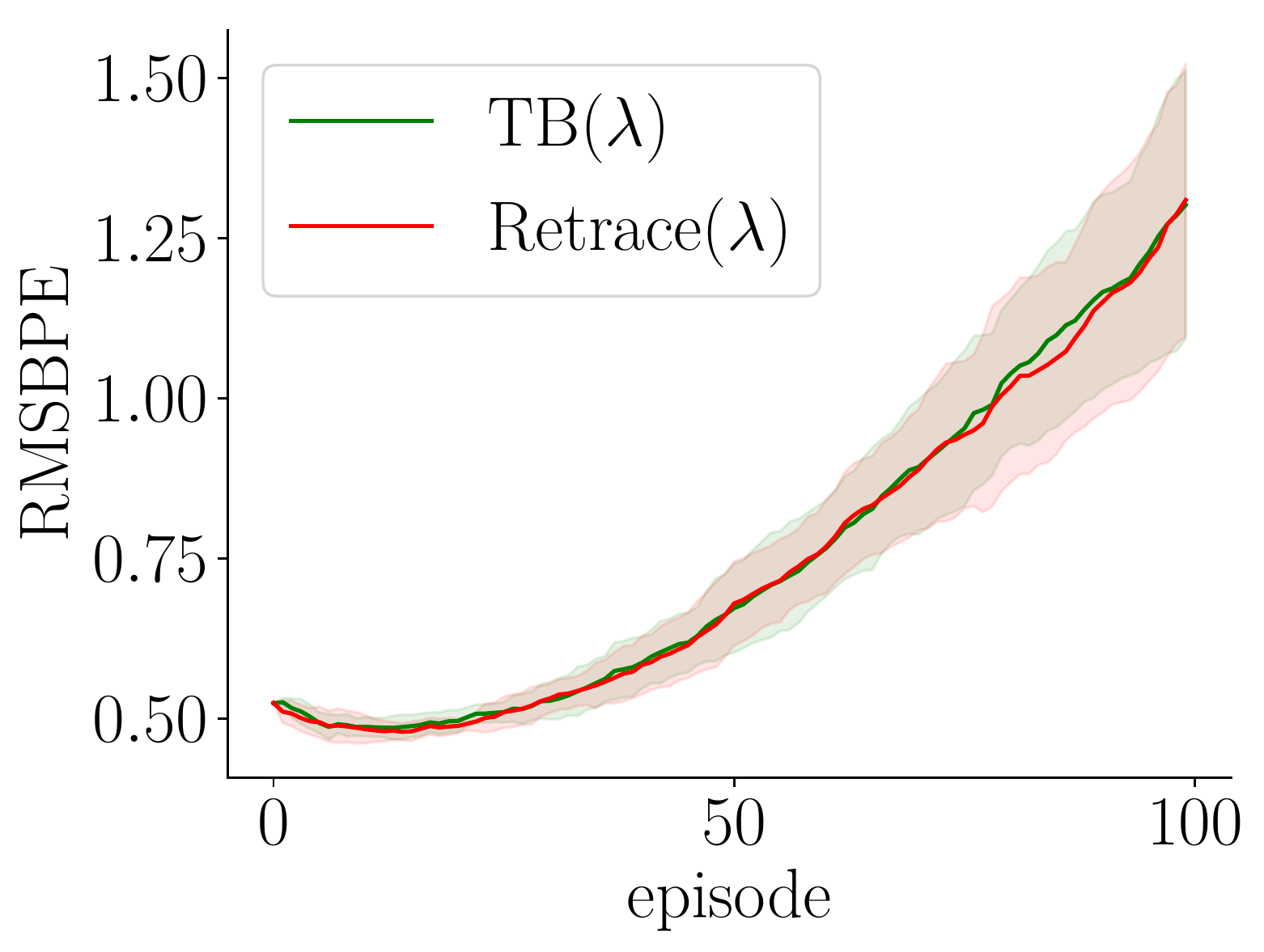}\hspace{2em}
\includegraphics[width=0.45\linewidth]{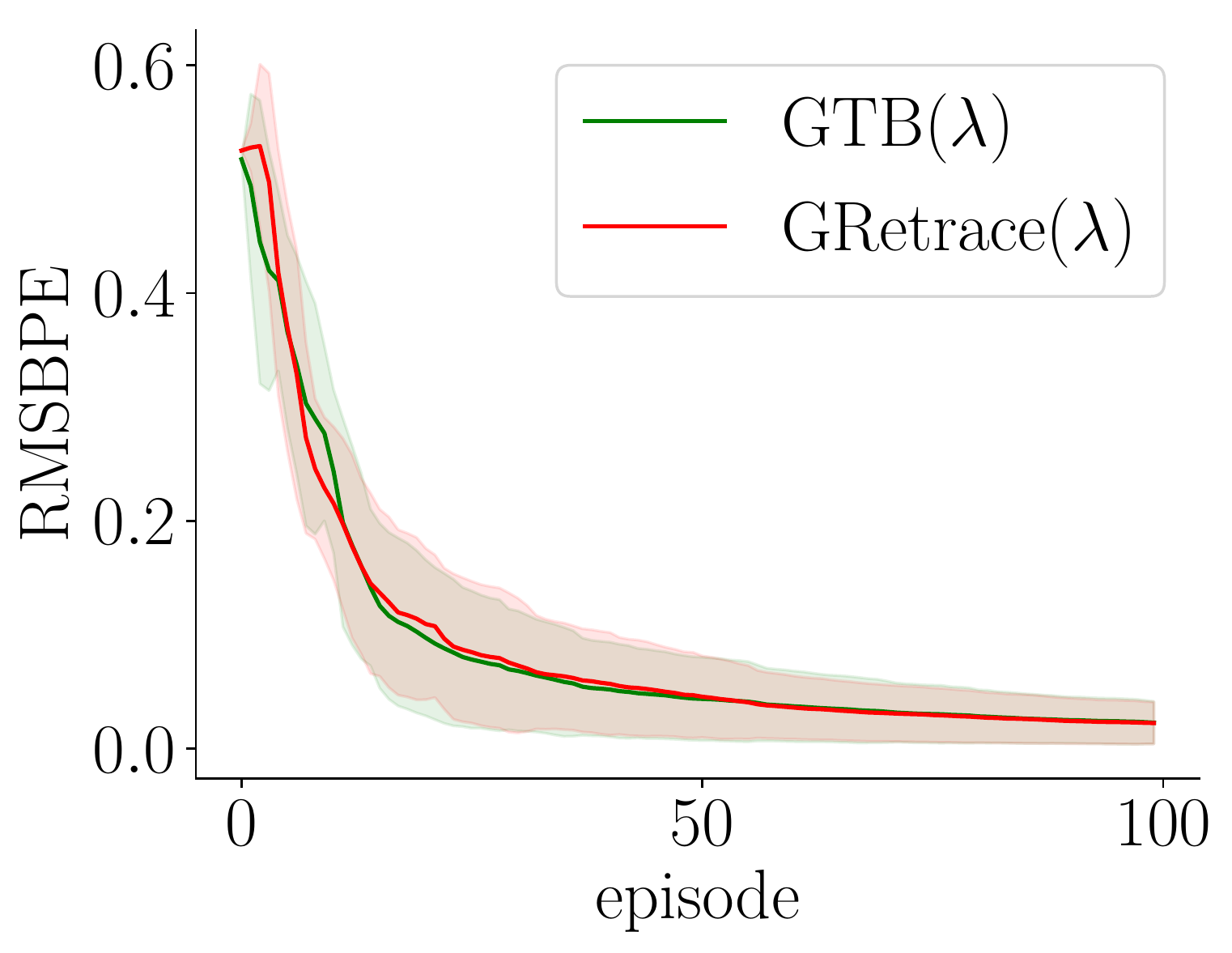}
\caption{\label{fig:2-state}In the $2$-states counterexample of section \ref{sect:offpolicyinstability} showing that the gradient-based TB and \textsc{Retrace} converge while TB and \textsc{Retrace} diverge.}
\end{figure}

\paragraph{Comparison with existing methods:} We also compared GTB($\lambda$) and G\textsc{Retrace}($\lambda$) with two recent state-of-the-art convergent off-policy algorithms for action-value estimation and function approximation: GQ($\lambda$)~\citep{maei2011gradient} and \textsc{AB-Trace}($\lambda$)~\citep{mahmood2017multi}. As in ~\citet{mahmood2017multi}, we also consider a policy evaluation task in the Mountain Car domain. In order to better understand the variance inherent to each method, we designed the target policy and behavior policy in such a way that the importance sampling ratios can be as large as $30$. We chose to describe state-action pairs by a $96$-dimensional vector of features derived by tile coding \citep{sutton1998introduction}. We ran each algorithm over all possible combinations of step-size values $(\alpha_k, \eta_k) \in [0.001, 0.005, 0.01, 0.05, 0.1]^2$ for $2000$ episodes and reported their normalized mean squared errors (NMSE):
\begin{equation*}
    \mathbf{NMSE}(\theta) = \frac{\|\Phi \theta - Q^{\pi} \|^2_{\Xi}}{ \| Q^{\pi} \|^2_{\Xi}}
\end{equation*}
where $Q^{\pi}$ is estimated by simulating the target policy and averaging the discounted cumulative rewards overs trajectories.
As \textsc{AB-Trace}($\lambda$) and G\textsc{Retrace}($\lambda$) share both the same operator, we can evaluate them using the empirical $\mathbf{MSPBE}= \frac{1}{2} ||\hat{A}\theta + \hat{b}||^2_{\hat{M}^{-1}}$ where $\hat{A}$, $\hat{b}$ and $\hat{M}$ are Monte-Carlo estimates obtained by averaging $\hat{A}_k$, $\hat{b}_k$ and $\hat{M}_k$ defined in proposition \ref{prop:eligibility} over $10000$ episodes.\\
Figure \ref{fig:MSPBE_min} shows that the best empirical $\mathbf{MSPBE}$ achieved by \textsc{AB-Trace}($\lambda$) and G\textsc{Retrace}($\lambda$) are almost identical across value of $\lambda$. This result is consistent with the 
fact that they both minimize the $\mathbf{MSPBE}$ objective function. However, significant differences can be observed when computing the 5th percentiles of NMSE (over all possible combination of step-size values) for different values of $\lambda$ in Figure \ref{fig:NMSE_quantile}. 
When $\lambda$ increases, the NMSE of GQ($\lambda$) increases sharply due to increased influence of importance sampling ratios. This clearly demonstrate the variance issues of GQ($\lambda$) in contrast with the other methods based on the \textsc{Tree Backup} and \textsc{Retrace} returns (that are not using importance ratios). For intermediate values of $\lambda$, \textsc{AB-Trace}($\lambda$) performs better but its performance is matched by G\textsc{Retrace}($\lambda$) and TB($\lambda$) for small and very large values of $\lambda$. In fact, \textsc{AB-Trace}($\lambda$) updates the function parameters $\theta$ as follows:
\begin{equation*}
    \theta_{k+1} = \theta_{k} - \alpha_k \left( \delta_k e_k - \Delta_k \right) 
\end{equation*}
where $\Delta_k \triangleq \gamma w_{k}^\top e_k (\mathbb{E}_{\pi}\phi(s_{k+1}, .) - \lambda \sum_a \kappa(s_k, a) \mu(a \cbar s_k) \phi(s_k, a))$ is a gradient correction term. When the instability is not an issue, the correction term could be very small and the update of $\theta$ would be essentially $\theta_{k+1} \sim \theta_{k} - \alpha_k \delta_k e_k$ so that $\theta_{k+1}$ follows the semi-gradient of the mean squared error $\| \Phi \theta - G^{\lambda}_k\|^2_{\Xi}$.\\
To better understand the errors of each algorithm and their robustness to step-size values, we propose the box plots shown in Figure \ref{fig:NMSE_boxplot}. Each box plot shows the distribution of NMSE obtained by each algorithm for different values of $\lambda$. NMSE distributions are computed over all possible combinations of step-size values. GTB($\lambda$) has the smallest variance as it scaled its return by the target probabilities which makes it conservative in its update even with large step-size values. G\textsc{Retrace}($\lambda$) tends to more more efficient than GTB($\lambda$) since it could benefit from full returns. The latter observation agrees with the tabular case of  \textsc{Tree Backup} and \textsc{Retrace}~\citep{munos2016safe}. Finally, we observe that \textsc{AB-Trace}($\lambda$) has lower error, but at the cost of increased variance with respect to step-size values.
\begin{figure}
\centering
\includegraphics[width=0.9\linewidth]{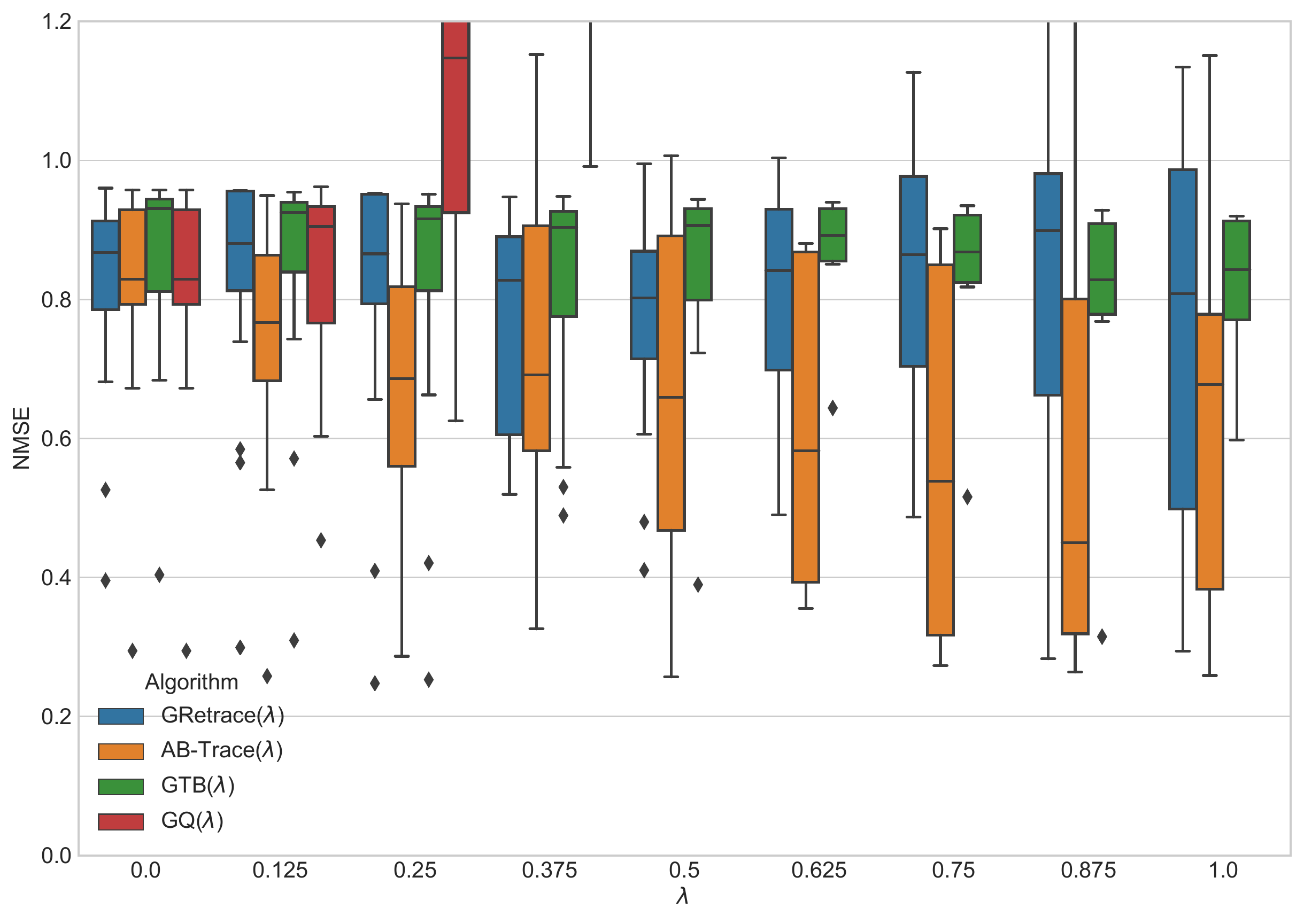}
\caption{\label{fig:NMSE_boxplot} Comparison of empirical performance of GQ($\lambda$), \textsc{AB-Trace}($\lambda$), G\textsc{Retrace}($\lambda$) and GTB($\lambda$) on an off-policy evaluation task in Mountain Car domain. Each box plot shows the distribution of the NMSE achieved by each algorithm after 2000 episodes for different values of $\lambda$. NMSE distributions are computed over all the possible combinations of step-size values $(\alpha_k, \eta_k) \in [0.001, 0.005, 0.01, 0.05, 0.1]^2$.}
\end{figure}

\begin{figure}[h]
\centering
\includegraphics[width=0.9\linewidth]{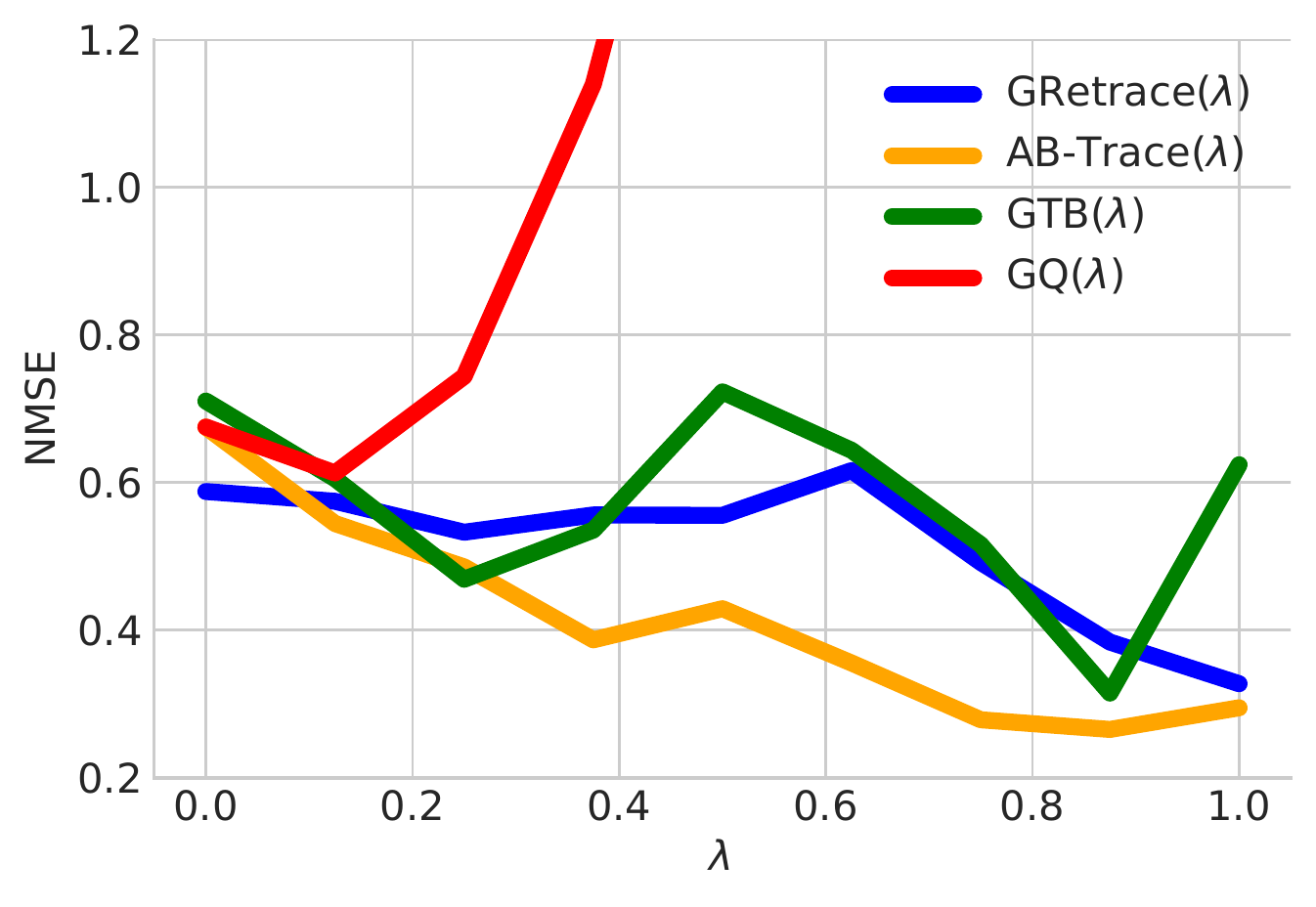}
\caption{\label{fig:NMSE_quantile} Each curves shows the 5th percentile of NMSE (over all possible combination of step-size values) achieved by each algorithm for different values of $\lambda$.}
\end{figure}

\begin{figure}
\centering
\includegraphics[width=0.9\linewidth]{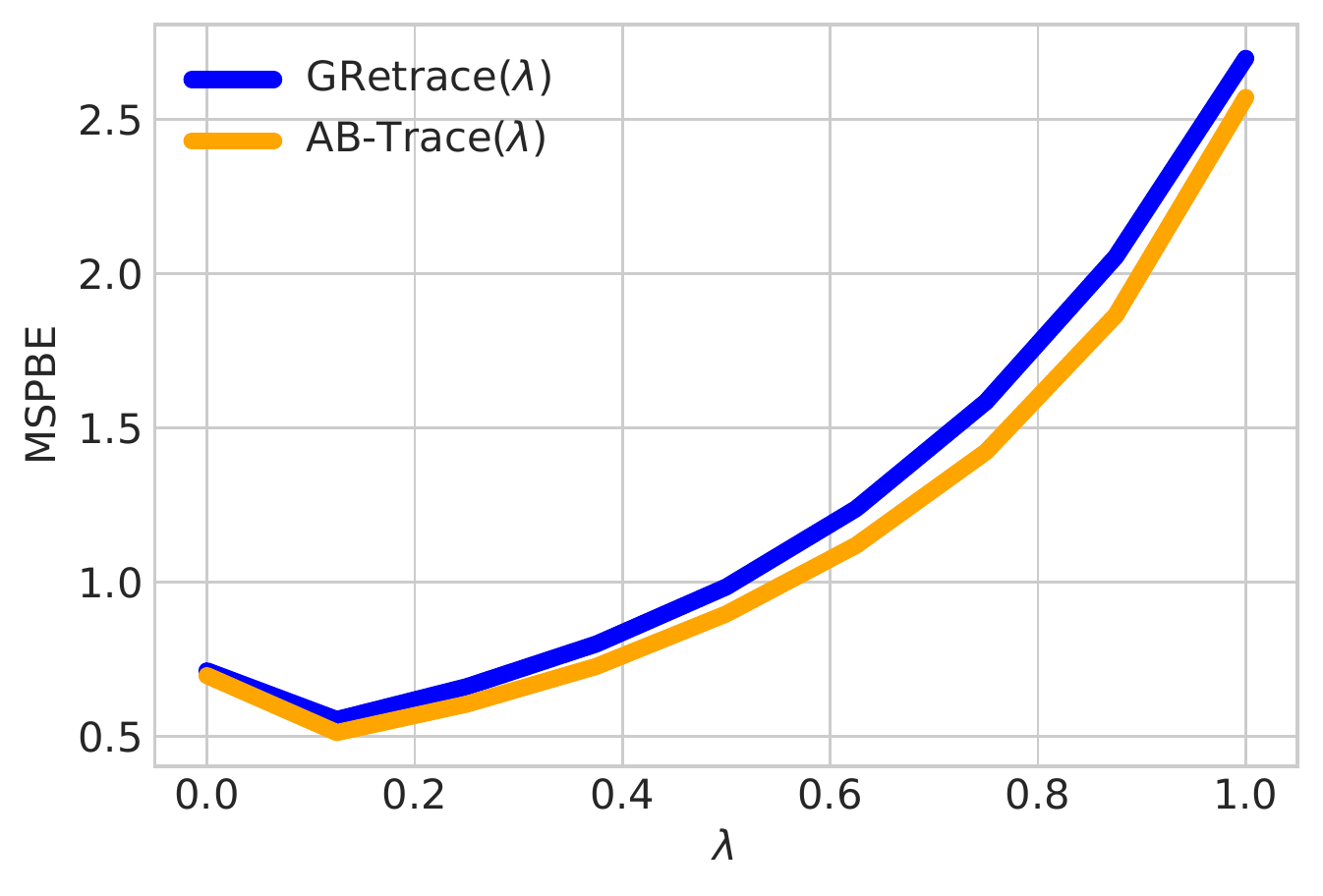}
\caption{\label{fig:MSPBE_min} Comparison between the best empirical MSPBE obtained by each algorithm for different values of $\lambda$. Only G\textsc{Retrace}($\lambda$) and \textsc{AB-Trace}($\lambda$) are showed here because the other algorithms do not have the same operators and hence not the same $\mathbf{MSPBE}$. Note that $\mathbf{MSPBE}$s depend on $\lambda$. Thus, $\mathbf{MSPBE}$s are not directly comparable across different values of $\lambda$. Both G\textsc{Retrace}($\lambda$) and \textsc{AB-Trace}($\lambda$) have very similar performances. \textsc{AB-Trace}($\lambda$) performs slightly better.}
\end{figure}

\section{Conclusion}
Our analysis highlighted for the first time the difficulties of combining the \textsc{Tree Backup} and \textsc{Retrace} algorithms with function approximation. We addressed these issues by formulating gradient-based algorithm versions of these algorithms which minimize the mean-square projected Bellman error. Using a saddle-point formulation, we were also able to provide convergence guarantees and  characterize the convergence rate of our algorithms GTB and G\textsc{Retrace}. We also developed a novel analysis method which allowed us to establish a $O(1/k)$ convergence rate without having to use Polyak averaging or projections (which might also make implementation more difficult). Furthermore, our proof technique is general enough that we were able to apply it to the existing GTD and GTD2 algorithms. Our experiments finally suggest that the proposed GTB($\lambda$) and G\textsc{Retrace} ($\lambda$) are robust to step-size selection and have less variance than both GQ($\lambda$) \citep{maei2011gradient} and \textsc{AB-Trace}($\lambda$) \citep{mahmood2017multi}.

\bibliographystyle{icml2018}
\bibliography{lib}

\begin{thebibliography}{41}
\providecommand{\natexlab}[1]{#1}
\providecommand{\url}[1]{\texttt{#1}}
\expandafter\ifx\csname urlstyle\endcsname\relax
  \providecommand{\doi}[1]{doi: #1}\else
  \providecommand{\doi}{doi: \begingroup \urlstyle{rm}\Url}\fi

\bibitem[Baird et~al.(1995)]{baird1995residual}
Baird, L. et~al.
\newblock Residual algorithms: Reinforcement learning with function
  approximation.
\newblock In \emph{Proceedings of the twelfth international conference on
  machine learning}, 1995.

\bibitem[Benzi \& Simoncini(2006)Benzi and Simoncini]{benzi2006eigenvalues}
Benzi, M. and Simoncini, V.
\newblock On the eigenvalues of a class of saddle point matrices.
\newblock \emph{Numerische Mathematik}, 2006.

\bibitem[Bertsekas(2011)]{bertsekas2011temporal}
Bertsekas, D.~P.
\newblock Temporal difference methods for general projected equations.
\newblock \emph{IEEE Transactions on Automatic Control}, 2011.

\bibitem[Bertsekas \& Tsitsiklis(1995)Bertsekas and
  Tsitsiklis]{bertsekas1995neuro}
Bertsekas, D.~P. and Tsitsiklis, J.~N.
\newblock Neuro-dynamic programming: an overview.
\newblock In \emph{Decision and Control, 1995., Proceedings of the 34th IEEE
  Conference on}. IEEE, 1995.

\bibitem[Borkar \& Meyn(2000)Borkar and Meyn]{Borkar2000}
Borkar, V.~S. and Meyn, S.~P.
\newblock The o.d.e. method for convergence of stochastic approximation and
  reinforcement learning.
\newblock \emph{{SIAM} Journal on Control and Optimization}, jan 2000.

\bibitem[Chen et~al.(2014)Chen, Lan, and Ouyang]{chen2014optimal}
Chen, Y., Lan, G., and Ouyang, Y.
\newblock Optimal primal-dual methods for a class of saddle point problems.
\newblock \emph{SIAM Journal on Optimization}, 2014.

\bibitem[Dalal et~al.(2017)Dalal, Szorenyi, Thoppe, and
  Mannor]{dalal2017finite}
Dalal, G., Szorenyi, B., Thoppe, G., and Mannor, S.
\newblock Finite sample analysis of two-timescale stochastic approximation with
  applications to reinforcement learning.
\newblock \emph{arXiv preprint arXiv:1703.05376}, 2017.

\bibitem[Defazio et~al.(2014)Defazio, Bach, and
  Lacoste-Julien]{defazio2014saga}
Defazio, A., Bach, F., and Lacoste-Julien, S.
\newblock Saga: A fast incremental gradient method with support for
  non-strongly convex composite objectives.
\newblock In \emph{Advances in neural information processing systems}, 2014.

\bibitem[Drazin(1958)]{Drazin1958}
Drazin, M.~P.
\newblock Pseudo-inverses in associative rings and semigroups.
\newblock \emph{The American Mathematical Monthly}, aug 1958.

\bibitem[Du et~al.(2017)Du, Chen, Li, Xiao, and Zhou]{du2017stochastic}
Du, S.~S., Chen, J., Li, L., Xiao, L., and Zhou, D.
\newblock Stochastic variance reduction methods for policy evaluation.
\newblock In \emph{International Conference on Machine Learning}, 2017.

\bibitem[Harutyunyan et~al.(2016)Harutyunyan, Bellemare, Stepleton, and
  Munos]{harutyunyan2016q}
Harutyunyan, A., Bellemare, M.~G., Stepleton, T., and Munos, R.
\newblock Q ($\lambda$) with off-policy corrections.
\newblock In \emph{International Conference on Algorithmic Learning Theory}.
  Springer, 2016.

\bibitem[Johnson \& Zhang(2013)Johnson and Zhang]{johnson2013accelerating}
Johnson, R. and Zhang, T.
\newblock Accelerating stochastic gradient descent using predictive variance
  reduction.
\newblock In \emph{Advances in neural information processing systems}, 2013.

\bibitem[Lakshminarayanan \& Szepesv{\'a}ri(2017)Lakshminarayanan and
  Szepesv{\'a}ri]{lakshminarayanan2017linear}
Lakshminarayanan, C. and Szepesv{\'a}ri, C.
\newblock Linear stochastic approximation: Constant step-size and iterate
  averaging.
\newblock \emph{arXiv preprint arXiv:1709.04073}, 2017.

\bibitem[Lin(1992)]{Lin1992}
Lin, L.-J.
\newblock Self-improving reactive agents based on reinforcement learning,
  planning and teaching.
\newblock \emph{Machine Learning}, may 1992.

\bibitem[Liu et~al.(2015)Liu, Liu, Ghavamzadeh, Mahadevan, and
  Petrik]{liu2015finite}
Liu, B., Liu, J., Ghavamzadeh, M., Mahadevan, S., and Petrik, M.
\newblock Finite-sample analysis of proximal gradient td algorithms.
\newblock In \emph{UAI}. Citeseer, 2015.

\bibitem[Macua et~al.(2015)Macua, Chen, Zazo, and Sayed]{macua2015distributed}
Macua, S.~V., Chen, J., Zazo, S., and Sayed, A.~H.
\newblock Distributed policy evaluation under multiple behavior strategies.
\newblock \emph{IEEE Transactions on Automatic Control}, 2015.

\bibitem[Maei(2011)]{maei2011gradient}
Maei, H.~R.
\newblock Gradient temporal-difference learning algorithms.
\newblock 2011.

\bibitem[Maei \& Sutton(2010)Maei and Sutton]{maei2010gq}
Maei, H.~R. and Sutton, R.~S.
\newblock Gq ($\lambda$): A general gradient algorithm for temporal-difference
  prediction learning with eligibility traces.
\newblock In \emph{Proceedings of the Third Conference on Artificial General
  Intelligence}, 2010.

\bibitem[Mahmood et~al.(2017)Mahmood, Yu, and Sutton]{mahmood2017multi}
Mahmood, A.~R., Yu, H., and Sutton, R.~S.
\newblock Multi-step off-policy learning without importance sampling ratios.
\newblock \emph{arXiv preprint arXiv:1702.03006}, 2017.

\bibitem[Mnih et~al.(2016)Mnih, Badia, Mirza, Graves, Lillicrap, Harley,
  Silver, and Kavukcuoglu]{mnih2016asynchronous}
Mnih, V., Badia, A.~P., Mirza, M., Graves, A., Lillicrap, T., Harley, T.,
  Silver, D., and Kavukcuoglu, K.
\newblock Asynchronous methods for deep reinforcement learning.
\newblock In \emph{International Conference on Machine Learning}, 2016.

\bibitem[Munos et~al.(2016)Munos, Stepleton, Harutyunyan, and
  Bellemare]{munos2016safe}
Munos, R., Stepleton, T., Harutyunyan, A., and Bellemare, M.
\newblock Safe and efficient off-policy reinforcement learning.
\newblock In \emph{Advances in Neural Information Processing Systems}, 2016.

\bibitem[Nemirovski et~al.(2009)Nemirovski, Juditsky, Lan, and
  Shapiro]{nemirovski2009robust}
Nemirovski, A., Juditsky, A., Lan, G., and Shapiro, A.
\newblock Robust stochastic approximation approach to stochastic programming.
\newblock \emph{SIAM Journal on optimization}, 2009.

\bibitem[Palaniappan \& Bach(2016)Palaniappan and
  Bach]{palaniappan2016stochastic}
Palaniappan, B. and Bach, F.
\newblock Stochastic variance reduction methods for saddle-point problems.
\newblock In \emph{Advances in Neural Information Processing Systems}, 2016.

\bibitem[Precup(2000)]{precup2000eligibility}
Precup, D.
\newblock Eligibility traces for off-policy policy evaluation.
\newblock \emph{Computer Science Department Faculty Publication Series}, 2000.

\bibitem[Precup et~al.(2001)Precup, Sutton, and
  Dasgupta]{precup2001eligibility}
Precup, D., Sutton, R.~S., and Dasgupta, S.
\newblock Off-policy temporal difference learning with function approximation.
\newblock In \emph{Proceedings of the Eighteenth International Conference on
  Machine Learning}, ICML '01, 2001.

\bibitem[Rosasco et~al.(2016)Rosasco, Villa, and V{\~u}]{rosasco2016stochastic}
Rosasco, L., Villa, S., and V{\~u}, B.~C.
\newblock Stochastic forward--backward splitting for monotone inclusions.
\newblock \emph{Journal of Optimization Theory and Applications}, 2016.

\bibitem[Sutton(2015)]{Sutton2015talk}
Sutton, R.~S.
\newblock Introduction to reinforcement learning with function approximation.
\newblock Tutorial Session of the Neural Information Processing Systems
  Conference, 2015.

\bibitem[Sutton \& Barto(1998)Sutton and Barto]{sutton1998introduction}
Sutton, R.~S. and Barto, A.~G.
\newblock \emph{Introduction to Reinforcement Learning}.
\newblock MIT Press, Cambridge, MA, USA, 1st edition, 1998.
\newblock ISBN 0262193981.

\bibitem[Sutton \& Barto(2018)Sutton and Barto]{SuttonBarto2017}
Sutton, R.~S. and Barto, A.~G.
\newblock \emph{Reinforcement Learning: An Introduction}.
\newblock MIT Press, Cambridge, MA, USA, 2nd edition, Near-final draft -- May
  27, 2018.

\bibitem[Sutton \& Tanner(2004)Sutton and Tanner]{Sutton2004}
Sutton, R.~S. and Tanner, B.
\newblock Temporal-difference networks.
\newblock In \emph{Advances in Neural Information Processing Systems 17}, 2004.

\bibitem[Sutton et~al.(1999)Sutton, Precup, and Singh]{sutton1999opt}
Sutton, R.~S., Precup, D., and Singh, S.
\newblock Between {MDPs} and semi-{MDPs}: A framework for temporal abstraction
  in reinforcement learning.
\newblock \emph{Artificial Intelligence}, aug 1999.

\bibitem[Sutton et~al.(2009{\natexlab{a}})Sutton, Maei, Precup, Bhatnagar,
  Silver, Szepesv{\'{a}}ri, and Wiewiora]{Sutton2009}
Sutton, R.~S., Maei, H.~R., Precup, D., Bhatnagar, S., Silver, D.,
  Szepesv{\'{a}}ri, C., and Wiewiora, E.
\newblock Fast gradient-descent methods for temporal-difference learning with
  linear function approximation.
\newblock In \emph{Proceedings of the 26th Annual International Conference on
  Machine Learning - ICML}. {ACM} Press, 2009{\natexlab{a}}.

\bibitem[Sutton et~al.(2009{\natexlab{b}})Sutton, Maei, Precup, Bhatnagar,
  Silver, Szepesv{\'a}ri, and Wiewiora]{sutton2009fast}
Sutton, R.~S., Maei, H.~R., Precup, D., Bhatnagar, S., Silver, D.,
  Szepesv{\'a}ri, C., and Wiewiora, E.
\newblock Fast gradient-descent methods for temporal-difference learning with
  linear function approximation.
\newblock In \emph{Proceedings of the 26th Annual International Conference on
  Machine Learning}. ACM, 2009{\natexlab{b}}.

\bibitem[Sutton et~al.(2009{\natexlab{c}})Sutton, Maei, and
  Szepesv{\'a}ri]{sutton2009convergent}
Sutton, R.~S., Maei, H.~R., and Szepesv{\'a}ri, C.
\newblock A convergent $ o (n) $ temporal-difference algorithm for off-policy
  learning with linear function approximation.
\newblock In \emph{Advances in neural information processing systems},
  2009{\natexlab{c}}.

\bibitem[Sutton et~al.(2011)Sutton, Modayil, Delp, Degris, Pilarski, White, and
  Precup]{sutton2011horde}
Sutton, R.~S., Modayil, J., Delp, M., Degris, T., Pilarski, P.~M., White, A.,
  and Precup, D.
\newblock Horde: A scalable real-time architecture for learning knowledge from
  unsupervised sensorimotor interaction.
\newblock In \emph{The 10th International Conference on Autonomous Agents and
  Multiagent Systems - Volume 2}, AAMAS '11, Richland, SC, 2011. International
  Foundation for Autonomous Agents and Multiagent Systems.

\bibitem[Sutton et~al.(2015)Sutton, Mahmood, and White]{sutton2015emphatic}
Sutton, R.~S., Mahmood, A.~R., and White, M.
\newblock An emphatic approach to the problem of off-policy temporal-difference
  learning.
\newblock \emph{The Journal of Machine Learning Research}, 2015.

\bibitem[Tsitsiklis et~al.(1997)Tsitsiklis, Van~Roy,
  et~al.]{tsitsiklis1997analysis}
Tsitsiklis, J.~N., Van~Roy, B., et~al.
\newblock An analysis of temporal-difference learning with function
  approximation.
\newblock \emph{IEEE transactions on automatic control}, 1997.

\bibitem[van Hasselt et~al.(2014)van Hasselt, Mahmood, and Sutton]{van2014off}
van Hasselt, H., Mahmood, A.~R., and Sutton, R.~S.
\newblock Off-policy td ($\lambda$) with a true online equivalence.
\newblock In \emph{Proceedings of the 30th Conference on Uncertainty in
  Artificial Intelligence, Quebec City, Canada}, 2014.

\bibitem[Wang \& Bertsekas(2013)Wang and Bertsekas]{wang2013stabilization}
Wang, M. and Bertsekas, D.~P.
\newblock Stabilization of stochastic iterative methods for singular and nearly
  singular linear systems.
\newblock \emph{Mathematics of Operations Research}, 2013.

\bibitem[Wang et~al.(2017)Wang, Chen, Liu, Ma, and Liu]{wang2017finite}
Wang, Y., Chen, W., Liu, Y., Ma, Z.-M., and Liu, T.-Y.
\newblock Finite sample analysis of the gtd policy evaluation algorithms in
  markov setting.
\newblock In \emph{Advances in Neural Information Processing Systems}, pp.\
  5510--5519, 2017.

\bibitem[Wang et~al.(2016)Wang, Bapst, Heess, Mnih, Munos, Kavukcuoglu, and
  de~Freitas]{wang2016sample}
Wang, Z., Bapst, V., Heess, N., Mnih, V., Munos, R., Kavukcuoglu, K., and
  de~Freitas, N.
\newblock Sample efficient actor-critic with experience replay.
\newblock \emph{arXiv preprint arXiv:1611.01224}, 2016.

\end{thebibliography}


\clearpage

\onecolumn  

\appendix

\section{Proof of Proposition \ref{prop:function_approx}}
We compute  $\E[A_k]$ and $\E[b_k]$ where expectation are over trajectories drawn by executing the behavior policy: $s_k, a_k, r_k, s_{k+1}, \ldots s_t, a_t, r_t, s_{t+1} \ldots$ where $s_k, a_k \sim d, \quad r_t =  r(s_t, a_t), \quad s_{t+1} \sim p(\cdot \cbar s_t, a_t)$. We note that under stationarity of $d$, $\E[A_k] = E[A_0] $ and $\E[b_k] = \E[b_0]$. Let $\theta, \theta' \in \R^d$ and let $Q = \Phi \theta$ and $Q' = \Phi \theta'$ their respective Q-functions. 
\begin{equation*}
\begin{aligned}
\theta'^{\top} \E[A_k] \theta & = \E\left[\sum_{t=0}^{\infty} (\lambda \gamma)^{t} \left(\prod_{i=1}^t \kappa_i\right) Q'(s_0, a_0) [\gamma \E_{\pi} Q(s_{t+1}, .) - Q(s_t, a_t) ]^\top \right] \\
& = \sum_{t=0}^{\infty} (\lambda \gamma)^t \E_{ \substack{s_{0:t+1} \\ a_{0:t}}}\left[ Q'(s_0, a_0)\left(\prod_{i=1}^t \kappa_i\right) 
[\gamma \E_{\pi} Q(s_{t+1}, .) - Q(s_t, a_t) ]^\top \right] \\
 & = \sum_{t=0}^{\infty} (\lambda \gamma)^t \E_{ \substack{s_{0:t} \\ a_{0:t}}}\left[ Q'(s_0, a_0) \left(\prod_{i=1}^t \kappa_i\right) ( \gamma \E_{s_{t+1}}[\E_{\pi} Q(s_{t+1}, .)| s_t, a_t] - Q(s_t, a_t)) \right] \\
 & = \sum_{t=0}^{\infty} (\lambda \gamma)^t \E_{ \substack{s_{0:t} \\ a_{0:t}}} \left[ Q'(s_0, a_0)\left(\prod_{i=1}^t \kappa_i\right) ( \gamma \sum_{s' \in \S} \sum_{a' \in \A} p(s'|s_t, a_t) \pi(a'|s') Q(s', a') - Q(s_t, a_t))\right] \\
& =  \sum_{t=0}^{\infty} (\lambda \gamma)^t\E_{ \substack{s_{0:t} \\ a_{0:t}}}\left[ Q'(s_0, a_0)\left(\prod_{i=1}^t \kappa_i\right) (\gamma P^{\pi}Q(s_t, a_t) - Q(s_t, a_t))\right] \\
& =  \sum_{t=0}^{\infty} (\lambda \gamma)^t \E_{ \substack{s_{0:t-1} \\ a_{0:t-1}}} \left[ Q'(s_0, a_0)
\left(\prod_{i=1}^{t-1} \kappa_i\right)  \sum_{s' \in \S} \sum_{a' \in \A} p(s'|s_{t-1}, a_{t-1}) \kappa(a', s') \mu(a'|s') (\gamma P^{\pi}Q(s', a') - Q(s', a'))\right] \\
& =  \sum_{t=0}^{\infty} (\lambda \gamma)^t  \E_{ \substack{s_{0:t-1} \\ a_{0:t-1}}}
\left[ Q'(s_0, a_0) \left(\prod_{i=1}^{t-1} \kappa_i\right) P^{\kappa \mu}(\gamma P^{\pi} - I)Q(s_{t-1}, a_{t-1}) \right] \\
& = \E_{s_0, a_0} \left[ Q'(s_0, a_0)\sum_{t=0}^{\infty} (\lambda \gamma)^t  ( P^{\kappa \mu})^t(\gamma P^{\pi} - I)Q(x_0, a_0)\right] \\
& =  \E_{s_0, a_0} \left[ Q'(s_0, a_0)(I - \lambda \gamma P^{\kappa \mu})^{-1}( \gamma P^{\pi} - I)Q(s_0, a_0) \right] \\
& = \sum_{s \in \S} \sum_{a \in \A} \xi(s, a) Q'(s, a)(I - \lambda \gamma P^{\kappa \mu})^{-1}( \gamma P^{\pi} - I)Q(s, a) \\
& = Q'^{\top} \Xi (I - \lambda \gamma P^{\kappa \mu})^{-1}( \gamma P^{\pi} - I)Q \\ 
\end{aligned}
\end{equation*}
So, $\theta'^{\top} \E[A_k] \theta = \theta'^{\top} \Phi^{\top} \Xi (I - \lambda \gamma P^{\kappa \mu})^{-1}(\gamma P^{\pi} - I) \Phi \theta \quad \forall \theta, \theta' \in \R^d$, which implies that:
\begin{equation*}
\E[A_k] = \Phi^{\top} \Xi (I - \lambda \gamma P^{\kappa \mu})^{-1}(P^{\pi} - I) \Phi
\end{equation*}
\begin{equation*}
\begin{aligned}
\theta^\top\E[b_k] & = \E[\sum_{t=0}^{\infty} (\lambda \gamma)^{t} \left(\prod_{i=1}^t \kappa_i\right) r_t Q(s_0, a_0)] 
= \sum_{t=0}^{\infty} (\lambda \gamma)^{t} \E_{ \substack{s_{0:t} \\ a_{0:t}} }\left[Q(s_0, a_0)  \left(\prod_{i=1}^t \kappa_i\right) r(s_t, a_t)  \right] \\
& = \sum_{t=0}^{\infty} (\lambda \gamma)^{t} \E_{ \substack{s_{0:t-1} \\ a_{0:t-1}} }\left[Q(s_0, a_0)  \left(\prod_{i=1}^{t-1} \kappa_i\right)  \sum_{s' \in \S} \sum_{a' \in \A} p(s'|s_{t-1}, a_{t-1}) \kappa(a', s') \mu(a'|s') r(s', a')  \right] \\
& = \sum_{t=0}^{\infty} (\lambda \gamma)^{t} \E_{ \substack{s_{0:t-1} \\ a_{0:t-1}} }\left[Q(s_0, a_0) \left (\prod_{i=1}^{t-1} \kappa_i\right)  P^{\kappa \mu}r(s', a')  \right] 
 = \E_{s_0, a_0}\left[ Q(s_0, a_0) (I - \lambda \gamma P^{\kappa \mu})^{-1} r(s_0, s_0)\right] \\
 & = \sum_{s \in \S} \sum_{a \in \A} \xi(s, a) Q(s, a)(I - \lambda \gamma P^{\kappa \mu})^{-1} r(s,a) = Q^{\top} \Xi (I - \lambda \gamma P^{\kappa \mu})^{-1} r
\end{aligned}
\end{equation*}
So, $\theta^\top\E[b_k]  = \theta^\top \Phi^{\top} \Xi (I - \lambda \gamma P^{\kappa \mu})^{-1} r \quad \forall \theta \in \R^d$, which implies that:
\begin{equation*}
\E[b_k] =  \Phi^{\top} \Xi (I - \lambda \gamma P^{\kappa \mu})^{-1} r
\end{equation*}

\section{Proof of Proposition \ref{prop:MSPBE}}
\begin{equation*}
\begin{aligned}
\mathbf{MSPBE}(\theta) & = \frac{1}{2}||\Pi^{\mu} \mathcal{R}(\Phi \theta)- \Phi \theta ||^2_{\Xi} = \frac{1}{2}||\Pi^{\mu} \left(\mathcal{R}(\Phi \theta)- \Phi \theta \right)||^2_{\Xi} \\
& = \frac{1}{2} \left(\Pi^{\mu} \left(\mathcal{R}(\Phi \theta)- \Phi \theta \right)\right)^\top \Xi \left(\Pi^{\mu} \left(\mathcal{R}(\Phi \theta)- \Phi \theta \right) \right) \\
& = \frac{1}{2}\left(\Phi^\top \Xi\left(\mathcal{R}(\Phi \theta)- \Phi \theta\right)\right)^\top (\Phi^\top \Xi \Phi)^{-1} \Phi^\top \Xi \left(\Phi (\Phi^\top \Xi \Phi)^{-1}\Phi^\top \Xi(\mathcal{R}(\Phi \theta)- \Phi \theta) \right) \\ 
& = \frac{1}{2} ||\Phi^\top \Xi\left(\R(\Phi \theta)- \Phi \theta\right) ||^2_{M^{-1}} \\
& = \frac{1}{2}  || \Phi^\top \Xi \left( (I - \lambda \gamma P^{\mu \pi})^{-1} (\T^{\pi} - \lambda \gamma P^{\mu \pi}) \Phi \theta - \Phi \theta \right)||^2_{M^{-1}} \\
& = \frac{1}{2}  || \Phi^\top \Xi(I - \lambda \gamma P^{\mu \pi})^{-1}(\gamma \mathit{P}^{\pi} - I) \Phi \theta + \Phi^\top \Xi(I - \lambda \gamma P^{\mu \pi})^{-1} r||^2_{M^{-1}} \\
& =  \frac{1}{2} || A \theta + b||^2_{M^{-1}}
\end{aligned}
\end{equation*}

\section{Proof of Proposition \ref{prop:eligibility}}
Let's show that $\mathbb{E}[\hat{A}_k] = A$. Let's $\Delta_t$ denotes $[\gamma \mathbb{E}_{\pi} \phi(s_{t+1}, .)^\top  - \phi(s_t, a_t)^\top]$
\begin{align*}
A  & = \E \left[ \sum_{t=k}^{\infty} (\lambda \gamma)^{t-k} \left(\prod_{i=k+1}^t \kappa_i\right) \phi(s_k, a_k) \Delta_t \right] \\
& =  \E \left[ \phi(s_k, a_k) \Delta_k + 
\sum_{t=k+1}^{\infty} (\lambda \gamma)^{t-k} \left(\prod_{i=k+1}^t \kappa_i\right) \phi(s_k, a_k) \Delta_t \right] \\
& = \E \left[ \phi(s_k, a_k) \Delta_k + \sum_{t=k}^{\infty} (\lambda \gamma)^{t-k+1} \left(\prod_{i=k+1}^{t+1} \kappa_i\right) \phi(s_{k}, a_{k}) \Delta_{t+1} \right] \\
&  = \E \left[ \phi(s_k, a_k) \Delta_k + \lambda \gamma \kappa(s_{k+1}, a_{k+1})\phi(s_k, a_k) \Delta_{k+1} +  \sum_{t=k+1}^{\infty} (\lambda \gamma)^{t-k+1} \left(\prod_{i=k+1}^{t+1} \kappa_i\right) \phi(s_{k}, a_{k}) \Delta_{t+1} \right] 
\end{align*}
\begin{align*}
\quad & \stackrel{(\star)}{=} \E \left[ \phi(s_k, a_k) \Delta_k + \lambda \gamma \kappa(s_{k}, a_{k})\phi(s_{k-1}, a_{k-1}) \Delta_{k} +  \sum_{t=k+1}^{\infty} (\lambda \gamma)^{t-k+1} \left(\prod_{i=k+1}^{t+1} \kappa_i\right) \phi(s_{k}, a_{k}) \Delta_{t+1} \right] \\
& = \E \left[ \Delta_k (\phi(s_k, a_k) + \lambda \gamma \kappa(s_{k}, a_{k})\phi(s_{k-1}, a_{k-1}) 
+ (\lambda \gamma)^2 \kappa(s_{k}, a_{k}) \kappa(s_{k-1}, a_{k-1})\phi(s_{k-2}, a_{k-2}) + ...) \right]  \\
& =  \E \left[\Delta_k \left(\sum_{i=0}^k  (\lambda \gamma)^{k-i} \left(\prod_{j=i+1}^{k} \kappa_j\right) \phi(x_i,a_i)\right) \right] \\
& = \E [\Delta_k e_k]  = \mathbb{E}[\hat{A}_k]
\end{align*}
we have used in the line ($\star$) the fact that 
$ \E[\kappa(s_{k+1}, a_{k+1})\phi(s_{k} a_{k}) \Delta_{k+1}] = \E[\kappa(s_{k}, a_{k})\phi(s_{k-1} a_{k-1}) \Delta_k]$ thanks to the stationarity of the distribution $d$.\\
we have also denote by $e_k$ the following vector:
\begin{equation*}
\begin{aligned}
e_k & = \sum_{i=0}^k  (\lambda \gamma)^{k-i} \left(\prod_{j=i+1}^{k} \kappa_j\right) \phi(s_i,a_i)\\ 
& = \lambda \gamma \kappa_k \left( \sum_{i=0}^{k-1}  (\lambda \gamma)^{k-1-i} \left(\prod_{j=i+1}^{k-1} \kappa_j\right) \phi(s_i,a_i)\right) + \phi(s_k, a_k) \\
& =  \lambda \gamma \kappa_k e_{k-1} + \phi(s_k, a_k)
\end{aligned}
\end{equation*}
Vector $e_k$ corresponds to the eligibility traces defined in the proposition. Similarly, we could show that $\mathbb{E}_{\mu}[\hat{b}_k] = b$.

\section{True on-line equivalence}
In \cite{van2014off}, the authors derived a true on-line update for GTD($\lambda$) that empirically performed better than GTD($\lambda$) with eligibility traces. Based on this work, we derive true on-line updates for our algorithm. The gradient off-policy algorithm was derived by turning the expected forward view into an expected backward view which can be sampled. In order to derive a true on-line update, we sample instead the forward view and then we turn the sampled forward view to an exact backward view using Theorem 1 in \cite{van2014off}.
If $k$ denotes the time horizon, we consider the sampled truncated interim forward return:

\begin{equation*}
\forall t < k, \quad Y_t^k = \sum_{i=t}^{k-1} (\lambda \gamma)^{i-t} \left(\prod_{j=t+1}^i \kappa_j\right) \delta_i
\end{equation*}

where $\delta_i = r_i + \theta_t^\top\E_{\pi}\phi(s_{t+1}, \cdot) - \theta_t^\top \phi(s_t, a_t)$, which gives us the sampled forward update of $\omega$:
\begin{equation} \label{eq:forward}
\forall k < t, \quad \omega_{t+1}^k = \omega_t^k + \alpha_t(Y_t^k - \phi(x_t, a_t)^\top \omega_t^k)\phi(x_t, a_t) 
\end{equation}
\begin{proposition} \label{prop:dutch}
For any k, the parameter $\omega_k^k$ defined by the forward view (\ref{eq:forward}) is equal to $\omega_k$ defined by the following backward view:
\begin{equation*}
\begin{aligned}
e^{\omega}_{-1} &= 0, \quad \forall k \geq 0 \\
e^{\omega}_k & = \lambda \gamma \kappa_{k} e^{\omega}_{k-1} + \alpha_k (1  - \lambda \gamma \kappa_k \phi(s_k, a_k)^\top e^{\omega}_{k-1} ) \phi(s_k, a_k) \\
\omega_{k+1} & = \omega_k + \delta_k e^{\omega}_k - \alpha_t \phi(s_k, a_k)^\top \omega_k   \phi(s_k, a_k)
\end{aligned}
\end{equation*}
\end{proposition}
\begin{proof}
The return's temporal difference  $Y_t^{k+1} - Y_t^k $ are related through:
\begin{equation*}
\begin{aligned}
\forall t<k, \quad Y_t^{k+1} -Y_t^k &=  \sum_{i=t}^{k} (\lambda \gamma)^{i-t} (\prod_{j=t+1}^i \kappa_j) \delta_i - \sum_{i=t}^{k-1} (\lambda \gamma)^{i-t} (\prod_{j=t+1}^i \kappa_j) \delta_i \\
& = (\lambda \gamma)^{k-t} \left(\prod_{j=t+1}^k \kappa_j\right) \delta_k \\
& = \lambda \gamma \kappa_{k+1} \left(  (\lambda \gamma)^{k-(t+1)} \left(\prod_{j=t+2}^k \kappa_j\right) \delta_k \right) \\
& =  \lambda \gamma \kappa_{k+1} \left( Y_{t+1}^{k+1} -Y_{t+1}^k \right)
\end{aligned}
\end{equation*}
We could then apply Theorem 1 of \cite{van2014off} that give us the following backward view:
\begin{equation*}
\begin{aligned}
e_0 &= \alpha_0 \phi(x_0, a_0) \\
e_t & = \lambda \gamma \kappa_{t} e_{t-1} + \alpha_t (1  - \lambda \gamma \kappa_k \phi(s_t, a_t)^\top e_{t-1} ) \phi(s_t, a_t) \quad \forall t > 0 \\
\omega_{t+1} & = \omega_t + ( Y_t^{t+1} -Y_t^t) e_t + \alpha_t (Y_t^t - \phi(s_t, a_t)^\top \omega_t)\phi(s_t, a_t)\\
& \stackrel{(\star)}{=}  \omega_t + \delta_t e_t - \alpha_t \phi(s_t, a_t)^\top \omega_t \phi(s_t, a_t)
\end{aligned}
\end{equation*}
We used in the line ($\star$) that $Y_t^{t+1} = \delta_t$ and $Y_t^t = 0$
\end{proof}

The resulting detailed procedure is provided in Algorithm \ref{algo:dutch}.

Note that when $\lambda$ is equal to zero, the Algorithm 1 and 2 both reduce to the same update:
\begin{equation*}
\begin{aligned}
\omega_{k+1}  & = \omega_k + \alpha_k ( \delta_k - \phi(s_k, a_k)^\top \omega_k) \phi(s_k, a_k) \\
\theta_{k+1} & = \theta_{k} - \alpha_k \phi(s_k, a_k)^\top w_{k}  (\gamma \mathbb{E}_{\pi}[\phi(s_{k+1}, .)] - \phi(s_k, a_k)])
\end{aligned}
\end{equation*}

\begin{algorithm}
\caption{\label{algo:dutch}Gradient Off-policy with eligibility/Dutch traces }
\begin{algorithmic}
\medskip
\item[\textbf{Given:}] target policy $\pi$, behavior policy $\mu$
\STATE Initialize $\theta_0$ and $\omega_0$
\FOR {n = 0 \ldots}	
	\STATE set $e^{\theta}_{-1} = e^{\omega}_{-1} = 0$
  	\FOR {k = 0 \ldots  end of episode}
    	\STATE Observe $s_k, a_k, r_k, s_{k+1}$ according to $\mu$
    	\STATE \textbf{Update traces}
    	\STATE $e_k = \lambda \gamma \kappa(s_k, a_k) e_{k-1} + \phi(s_k, a_k)$
        \STATE \textbf{Update Dutch traces}
		\STATE $e^{\omega}_k = \lambda \gamma \kappa_{k} e^{\omega}_{k-1} + \alpha_k \left(1  - \lambda \gamma \kappa_k \phi(s_k, a_k)^\top e^{\omega}_{k-1} \right) \phi(s_k, a_k)$
    	\STATE \textbf{Update parameters}
    	\STATE $\delta_k = r_k + \gamma \theta_{k}^\top \mathbb{E}_{\pi}\phi(s_{k+1}, .) - \theta_{k}^\top\phi(s_k, a_k)$
    	\STATE $\omega_{k+1}  = \omega_k + \delta_k e^{\omega}_k - \alpha_k \phi(s_k, a_k)^\top \omega_k \phi(s_k, a_k)$
    	\STATE $\theta_{k+1} = \theta_{k} - \alpha_k \omega_{k}^\top e_k \left(\gamma \mathbb{E}_{\pi}[\phi(s_{k+1}, .)] - \phi(s_k, a_k)\right)$
  	\ENDFOR

\ENDFOR
\end{algorithmic}
\end{algorithm}

\section{Convergence Rate Analysis}
Let's recall the key quantities defined in the main article:
\begin{equation*}
\rho \triangleq \lambda_{\max}(A^{\top}M^{-1}A), \quad
\delta  \triangleq \lambda_{\min}(A^{\top} M^{-1}A ), \quad
L_G \triangleq \| \E \left[ \hat{G}_k^{\top} \hat{G}_k\cbar \mathcal{F}_k \right] \|
\end{equation*}
We will make use of spectral properties of the matrix $G$ provided in the appendix A of  \citep{du2017stochastic}.
it was shown that if we set $\beta = \frac{8 \rho}{\lambda_{\min}(M)}$, the matrix $G$ is diagonalizable with all its eigenvalues real and positive. It is a straightforward application of result from \citep{benzi2006eigenvalues} \\
Moreover, it was proved that $G$ can be written as: 
$G= Q \Lambda Q^{-1}$ where $\Lambda$ is a diagonal matrix whose diagonal entries are the eigenvalues of $G$ and $Q$ consists of it eigenvectors as columns such that the condition number of Q is upper bounded by the one of $M$ as follows:
\begin{equation*}
c(Q)^2 \leq 8 c(M)
\end{equation*}
Finally, the paper showed upper and lower bounds for the eigenvalues of G:
\begin{align*}
\lambda_{\max}(G) & \leq 9 c(M) \rho \\
\lambda_{\min}(G) &  \geq \frac{8}{9} \delta
\end{align*}
Let's recall our updates:
\begin{equation*}
    z_{k+1} = z_k - \alpha_k (\hat{G}_k z_k - \hat{g}_k)
\end{equation*}
By subtracting $z^{\star}$ from both sides on the later equation and using the optimality condition $G z^{\star}+ g = 0$:
\begin{equation}
\Delta_{k+1} = \Delta_k - \alpha_k G \Delta_k + \alpha_k \left[ G z_k - g - (\hat{G}_k z_k - \hat{g}_k)\right] \quad 
\text{where} \quad \Delta_k \triangleq z_k - z^{\star}
\end{equation}
By multiplying both sides by $Q^{-1}$ and using the fact that $Q^{-1}G = \Lambda Q^{-1}$:
\begin{align*}
Q^{-1}\Delta_{k+1} & = Q^{-1}\Delta_k - \alpha_k Q^{-1} G \Delta_k + \alpha_k Q^{-1}\left[ G z_k - g - (\hat{G}_k z_k - \hat{g}_k)\right] \\
& = (I - \alpha_k\Lambda) Q^{-1}\Delta_k + \alpha_k Q^{-1}\left[ G z_k - g - (\hat{G}_k z_k - \hat{g}_k) \right]
\end{align*}
\begin{align*}
\E \left[ \Big\|Q^{-1}\Delta_{k+1} \Big\|^2 \cbar \mathcal{F}_{k-1} \right] &=  \E \left[ 
\Big\| (I - \alpha_k) Q^{-1}\Delta_k + \alpha_k Q^{-1}\left[ G z_k  -g - (\hat{G}_k z_k - \hat{g}_k) \right] \Big\|^2
\cbar \mathcal{F}_{k-1} \right] \\
& = \E \left[ \Big\|   (I - \alpha_k\Lambda) Q^{-1}\Delta_k \Big\|^2 \cbar \mathcal{F}_{k-1} \right] \\
& \phantom{{}=1}  + 2\E \left[  \Big\langle  (I - \alpha_k) Q^{-1}\Delta_k,  \alpha_k Q^{-1}\left[ G z_k - g - (\hat{G}_k z_k - \hat{g}_k) \right] \Big\rangle \cbar \mathcal{F}_{k-1} \right] \notag \\
 & \phantom{{}=1} +\alpha_k^2\E \left[ \Big\| Q^{-1}\left[ G z_k - g - (\hat{G}_k z_k - \hat{g}_k) \right] \Big\|^2
\cbar \mathcal{F}_{k-1} \right] \notag \\
& =  \Big\|   (I - \alpha_k\Lambda) Q^{-1}\Delta_k \Big\|^2+ \alpha_k^2\E \left[ \Big\| Q^{-1}\left[ G z_k - g - (\hat{G}_k z_k - \hat{g}_k) \right]\Big\|^2
\cbar \mathcal{F}_{k-1} \right]  \\
& \leq \Big\|  I - \alpha_k\Lambda \Big\|^2 \Big\| Q^{-1}\Delta_k \Big\|^2 + \alpha_k^2 \E \left[ \| Q^{-1}  (\hat{G}_k z_k - \hat{g}_k) \|^2 \cbar \mathcal{F}_{k-1} \right] \\
& =  \Big\|  I - \alpha_k\Lambda \Big\|^2 \Big\| Q^{-1}\Delta_k \Big\|^2 + \alpha_k^2 \E \left[ \Big\| Q^{-1}  ( \hat{G}_k\Delta_k + \hat{G}_k z^{\star} - \hat{g}_k) \|^2 \cbar \mathcal{F}_{k-1} \right] \\
& \leq  \Big\|  I - \alpha_k\Lambda \Big\|^2 \Big\| Q^{-1}\Delta_k \Big\|^2 + 2 \alpha_k^2 \E \left[ \Big\| Q^{-1} \hat{G}_k\Delta_k \Big\|^2 \cbar \mathcal{F}_{k-1} \right]  + 2 \alpha_k^2 \E \left[ \Big\|  Q^{-1}(\hat{G}_k z^{\star} + \hat{g}_k) \Big\|^2 \cbar \mathcal{F}_{k-1} \right]
\end{align*}
we use in the third line the fact that  $\E \left[ \hat{G}_k \cbar \mathcal{F}_{k-1}  \right] = G$ and $\E \left[ \hat{g}_{k-1} \cbar \mathcal{F}_{k-1}  \right] = g$.
\begin{align*}
\|  I - \alpha_k\Lambda \|^2 & = \max\{ |1 - \alpha_k \lambda_{\min}(G)|^2, | 1 - \alpha_k \lambda_{\max}(G) |^2 \} \\
& \leq 1 - 2 \alpha_k \lambda_{\min} + \alpha_k^2 \lambda_{\max}^2 \\
& \leq 1 - 2 \alpha_k \frac{8}{9}\delta + \alpha_k^2 9^2 c(M)^2 \rho^2 \\
& \leq 1 -2 \alpha_k \delta' + \alpha_k^2 9^2 c(M)^2 \rho^2
 \quad \text{where} \quad \delta' \triangleq \frac{8}{9}\delta
\end{align*}
\begin{align*}
\E \left[ \Big\| Q^{-1} \hat{G}_k\Delta_k \Big\|^2 \cbar \mathcal{F}_{k-1} \right] 
& \leq \| Q^{-1}\|^2  \E \left[ \Big\|\hat{G}_k\Delta_k \Big\|^2 \cbar \mathcal{F}_{k-1} \right] \\
& = \| Q^{-1}\|^2 \E \left[ \Delta_k^{\top} \hat{G}_k^{\top} \hat{G}_k\Delta_k 
\cbar \mathcal{F}_{k-1} \right]  \\
& =  \| Q^{-1}\|^2 \Delta_k^{\top} \E \left[ \hat{G}_k^{\top} \hat{G}_k\cbar \mathcal{F}_{k-1} \right] \Delta_k    \\
& \leq  \| Q^{-1}\|^2 \Big\| \E \left[ \hat{G}_k^{\top} \hat{G}_k\cbar \mathcal{F}_k \right] \Big\|^2 \Delta_k^{\top} \Delta_k \\
& \leq  \| Q^{-1}\|^2  L_G \| \Delta_k\|^2 \\
& = \| Q^{-1}\|^2  L_G \|Q Q^{-1} \Delta_k\|^2 \\
& \leq \| Q^{-1}\|^2 \|Q\|^2  L_G \|Q^{-1} \Delta_k\|^2 \\
& \leq c(Q)^2 L_G \|Q^{-1} \Delta_k\|^2
\end{align*}
So, we have:

\begin{equation*}
\E \left[ \| Q^{-1} \Delta_{k+1} \|^2\right] 
\leq (1 -2 \alpha_k \delta' + \alpha_k^2 9^2 c(M)^2 \rho^2 + 16 \alpha_k^2 c(M) L_G)  \E \left[\| Q^{-1}\Delta_k\|^2 \right] + 2 \alpha_k^2  \| Q^{-1}\|^2 \E \left[ \|  \hat{G}_k z^{\star} - \hat{g}_k) \|^2 \right] 
\end{equation*}

By selecting $\alpha_k = \frac{2 \delta'}{\delta'^2 (k+2) + 2\times 9^2 c(M)^2 \rho^2
+ 32 c(M) L_G} 
= \frac{2 \delta'}{\delta'^2 (k+2) + \zeta}$ with $\zeta =2 \times 9^2 c(M)^2 \rho^2
+ 32 c(M) L_G$, we get:
\begin{align*}
\E \left[ \| Q^{-1} \Delta_{k+1} \|^2\right] & \leq (1 - \delta' \alpha_k) \E \left[\| Q^{-1} \Delta_k\|^2 \right] + 2 \alpha_k^2  \| Q^{-1}\|^2 \E \left[ \| \hat{G}_k z^{\star} - \hat{g}_k \|^2 \right] \\
& = \frac{\delta'^2 k + \zeta }{\delta'^2(k+2) + \zeta} \E \left[ \| Q^{-1} \Delta_{k} \|^2\right] + \frac{8 \delta'^2}{(\delta'^2(k+2) + \zeta)^2} \| Q^{-1}\|^2 \E \left[ \|  \hat{G}_k z^{\star} + \hat{g}_k) \|^2 \right] \\
& \leq \left( \prod_{i = 0}^k \frac{\delta'^2 i + \zeta }{\delta'^2(i+2) + \zeta} \right) \E \left[ \| Q^{-1} \Delta_{0} \|^2\right]   \notag \\
& \phantom{{}=1} + 8 \delta'^2 \sum_{i = 0}^k \left( \prod_{j = i}^k \frac{\delta'^2 j + \zeta }{\delta'^2(j+2) + \zeta} \right) \frac{1}{(\delta'^2(i+2) + \zeta)^2} 
\| Q^{-1}\|^2 \E \left[ \|  \hat{G}_i z^{\star} + \hat{g}_i) \|^2 \right] \\
& = \frac{\zeta  (\delta'^2 + \zeta) }{  (\delta'^2(k+1) + \zeta) ( (\delta'^2(k+2) + \zeta)} \E \left[ \| Q^{-1} \Delta_{0} \|^2\right]  \notag \\
&  \phantom{{}=1} + 8 \delta'^2 \sum_{i = 0}^k \frac{(\delta'^2(i+1) + \zeta)  (\delta'^2i + \zeta) }{  (\delta'^2(k+1) + \zeta) ( (\delta'^2(k+2) + \zeta)}
\frac{1}{(\delta'^2(i+2) + \zeta)^2} 
\| Q^{-1}\|^2 \E \left[ \|  \hat{G}_i z^{\star} - \hat{g}_i\|^2 \right]  \\
& \leq \frac{\zeta  (\delta'^2 + \zeta) }{  (\delta'^2(k+1) + \zeta) ( \delta'^2(k+2) + \zeta)} \E \left[ \| Q^{-1} \Delta_{0} \|^2\right]  \notag \\
&  \phantom{{}=1} + 8 \delta'^2 \sum_{i = 0}^k \frac{1}{  (\delta'^2(k+1) + \zeta) ( \delta'^2(k+2) + \zeta)}
\| Q^{-1}\|^2 \E \left[ \|  \hat{G}_i z^{\star} - \hat{g}_i \|^2 \right]  \\
& \leq \frac{(\delta' + \zeta)^2}{(\delta'^2 (k+1) + \zeta)^2 } \E \left[ \| Q^{-1} \Delta_{0} \|^2\right] \\
&  \phantom{{}=1} + 8 \frac{\delta'^2 (k+1) }{  (\delta'^2(k+1) + \zeta) ( \delta'^2(k+2) + \zeta)} \| Q^{-1}\|^2
\sup_{i = 0 \ldots k} \E \left[ \|  \hat{G}_i z^{\star} + \hat{g}_i) \|^2 \right] \\
& \leq \frac{(\delta' + \zeta)^2}{(\delta'^2 (k+1) + \zeta)^2 }\E \left[ \| Q^{-1} \Delta_{0} \|^2\right] 
+ \frac{8}{  (\delta'^2(k+1) + \zeta) } \| Q^{-1}\|^2
\sup_{i = 0 \ldots k} \E \left[ \|  \hat{G}_i z^{\star} - \hat{g}_i \|^2 \right] \\
&\leq \frac{(\delta' + \zeta)^2 \| Q^{-1}\|^2}{(\delta'^2 (k+1) + \zeta)^2 }\E \left[ \| \Delta_{0} \|^2\right] 
+ \frac{8 \sigma^2 \| Q^{-1}\|^2}{  (\delta'^2(k+1) + \zeta) } (1 + \| z^{\star}\|^2)
\end{align*}

Moreover, we have  
$\E \left[ \| \Delta_{k+1} \|^2\right] = \E \left[ \| Q Q^{-1}\Delta_{k+1} \|^2\right]
    \leq \| Q\|^2 \E \left[ \|Q^{-1}\Delta_{k+1} \|^2\right] $. Then, we get:
\begin{align*}
    \E \left[ \| \Delta_{k+1} \|^2\right] 
    & \leq \frac{(\delta' + \zeta)^2 c(Q)^2}{(\delta'^2 (k+1) + \zeta)^2 }\E \left[ \| \Delta_{0} \|^2\right] 
+ \frac{8 \sigma^2 c(Q)^2}{  (\delta'^2(k+1) + \zeta) } (1 + \| z^{\star}\|^2) \\
& \leq \frac{8 (\delta' + \zeta)^2 c(M)}{(\delta'^2 (k+1) + \zeta)^2 }\E \left[ \| \Delta_{0} \|^2\right] 
+ \frac{8^2 \sigma^2 c(M)}{  (\delta'^2(k+1) + \zeta) } (1 + \| z^{\star}\|^2) \\
& = \frac{8 9^2 (8\delta + 9\zeta)^2 c(M)}{(8^2\delta^2 (k+1) + 9^2\zeta)^2 }\E \left[ \| \Delta_{0} \|^2\right] 
+ \frac{9^2 \times 8^2 \sigma^2 c(M)}{  (8^2\delta^2(k+1) + 9^2\zeta) } (1 + \| z^{\star}\|^2) \\
& = 9^2 \times 8 c(M) \Big\{ \frac{(8\delta + 9\zeta)^2 \E \left[ \| \Delta_{0} \|^2\right] }{(8^2\delta^2 (k+1) + 9^2\zeta)^2 } 
+ \frac{8 \sigma^2 (1 + \| z^{\star}\|^2)}{  (8^2\delta^2(k+1) + 9^2\zeta) } \Big\}
\end{align*}
The overall convergence rate is then equal to $O(1/k)$.
\end{document}